%% file: 17mrs-paper-with-appendices.tex
\documentclass[letterpaper, 10 pt, conference]{ieeeconf}  % Comment this line out if you need a4paper
\IEEEoverridecommandlockouts                              % This command is only needed if
\usepackage{amsfonts}       % An extended set of fonts for use in mathematics
\usepackage{amsmath}        % provides features to facilitate writing math formulas
\usepackage{amssymb}        % provides AMS fonts and symbols ex. \mathbb{...}
\usepackage{array}          % extended implementations of array and tabular
\usepackage{bigints}        % allows for better sizing of integrals
\usepackage{cite}           % cleans up citations
\usepackage{colortbl}       % allows rows and columns to be colored, and even individual cells.
\usepackage{epsfig}         % allows inclusion of a figure in a separate postscript file
\usepackage[outdir=./]{epstopdf}
\usepackage{enumerate}      % gives enumerate an optional "style" argument
\usepackage{graphicx}       % enhanced support for graphics
\usepackage{lmodern}        % Try to suppress Font warnings
\usepackage{marvosym}       % better @ symbol
\usepackage{mathrsfs}       % Support use of the Raph Smith's Formal Script font in mathematics
\usepackage{multicol}       % Columns for appendix
\usepackage{multirow}       % allows a cell to span multiple rows in tabular environment
\usepackage{overpic}        % Add text on figures
\usepackage{pgfplots}       % Plot
\usepackage{placeins}       % for \FloatBarrier
\usepackage{pstricks}       % provides macros that allow the inclusion of postscript drawings 
\usepackage{pst-node}       % enables the user to connect information, and to place labels,
\usepackage{tabularx}       % requires: array ; implements tabular so that the widths of certain 
\usepackage[flushleft]{threeparttable} % Table notes
\usepackage{tikz}           % Figures
\usepackage{times}          % assumes new font selection scheme installed
\usepackage[textsize=footnotesize]{todonotes}
\usepackage[normalem]{ulem} % provides fancy underlining \ul
\usepackage{upgreek}        % provides Uppercase Greek letters
\usepackage{varwidth}       % similar to minipage, box may get a narrower “natural” width.
\usepackage{xcolor}	        % driver independent color extensions for latex and pdflatex
\usepackage{xspace}         % looks at what comes after it in the command stream and decides 
\usepackage{xstring}        % macros for manipulating strings
\usepackage[caption=false]{subfig} % Get subfig functionality without redefining the caption command and breaking everything in ieeeconf

\usepackage[vlined,ruled,linesnumbered]{algorithm2e}
\usepackage{algpseudocode}

\usepackage{etoolbox}
\makeatletter
\renewcommand\@makefntext[1]{%
  \hbadness=10000
  \leavevmode
  \rlap{\hss\@thefnmark\enspace}%
  #1%
}
\makeatother

\usepackage{amsthm}         % better definition of newtheorem
\pgfplotsset{compat=1.13}
\usetikzlibrary{
    shapes,
    shapes.geometric,
    decorations.markings}
\definecolor{node_fill}{RGB}{160, 221, 232}
\newtheorem{problem}{Problem} 
\newtheorem{proposition}{Proposition}[section]
\newtheorem{lemma}{Lemma}[section]

\newtheorem{corollary}{Corollary}[section]
\newtheorem{theorem}{Theorem}[section]

\theoremstyle{definition}

\theoremstyle{remark}
\newtheorem*{remark}{Remark}

\SetKwProg{Fn}{Function}{}{}
\SetKwComment{Comment}{$\triangleright$\ }{}
\newcommand{\glsr}{{\tt GLSR}\xspace}
\newcommand{\lsr}{{\tt LSR}\xspace}
\newcommand{\clsr}{{\tt C-LSR}\xspace}
\newcommand{\pmg}{{\tt PMG}\xspace}
\newcommand{\pmt}{{\tt PMT}\xspace}
\newcommand{\objects}{{\mathcal{O}}}
\newcommand{\arrangement}{\pi}

\newcommand{\tpap}{pop-and-push\xspace}
\newcommand{\tpapa}{pop-and-push action\xspace}

\title{Efficient, High-Quality Stack Rearrangement}

\author{%
Shuai D. Han\quad Nicholas M. Stiffler\quad Kostas E. Bekris\quad Jingjin Yu% <-this % stops a space
\thanks{S. D. Han, N. M. Stiffler, K. E. Bekris, and J. Yu 
        are with the Department of Computer Science, Rutgers, the State University of New Jersey, Piscataway, NJ, USA. 
        \{{\tt shuai.han, nick.stiffler, kostas.bekris, jingjin.yu}\}\hspace*{.25em}\MVAt \hspace*{.25em}rutgers.edu%
}%
\thanks{This work is supported by NSF awards IIS-1617744, IIS-1451737
  and CCF-1330789, as well as internal support by Rutgers
  University. Opinions or findings expressed here do not
  reflect the views of the sponsor.}% <-this % stops a space
}
\begin{document}

\maketitle

\thispagestyle{empty}
\pagestyle{empty}

%%%%%%%%%%%%%%%%%%%%%%%%%%%%%%%%%%%%%%%%%%%%%%%%%%%%%%%%%%%%%%
%% Begin Main Content
%%%%%%%%%%%%%%%%%%%%%%%%%%%%%%%%%%%%%%%%%%%%%%%%%%%%%%%%%%%%%%
%%%%%%%%%%%%%%%%%%%%%%%%%%%%%%%%%%%%%%%%%%%%%%%%%%%%%%%%%%%%%%

%% OUTLINE
%% ----------
%% 0. Abstract
%% 1. Introduction 
%% 2. Related work
%% 3. Problem Formulation
%% 4. Computational Complexity
%% 5. Optimal Solvers
%% 6. Experimental Results
%% 7. Conclusion

%%%%%%%%%%%%%%%%%%%%%%%%%%%%%%%%%%%%%%%%%%%%%%%%%%%%%%%%%%%%%%
% Since this is an evolving draft, I've removed the numerical
% prefix attached to the filenames since changes to the 
% composition of the paper at this stage could require 
% renaming several files. 

% TODO: Will add prefix as the outline of the paper stabilizes.
%%%%%%%%%%%%%%%%%%%%%%%%%%%%%%%%%%%%%%%%%%%%%%%%%%%%%%%%%%%%%%
\input{abstract}        % Abstract
\input{intro}           % Introduction
\input{related-work}    % Related work
\input{problem}         % Problem Formulation
\input{feasibility}     % Feasibility
\input{bounds}          % Bounds
%\input{complexity}      % Computational complexity and sub-optimal algorithm
\input{optimal}         % Optimal Solvers
\input{experiments}
% Experiments
\input{conclusion}      % Conclusion

%\section*{Acknowledgements}

%%%%%%%%%%%%%%%%%%%%%%%%%%%%%%%%%%%%%%%%%%%%%%%%%%%%%%%%%%%%%%%%
%% Bibliography
%%%%%%%%%%%%%%%%%%%%%%%%%%%%%%%%%%%%%%%%%%%%%%%%%%%%%%%%%%%%%%%%
{\small
\bibliographystyle{abbrv}
\bibliography{file}{17mrs-paper-with-appendices.bbl}
}

%%%%%%%%%%%%%%%%%%%%%%%%%%%%%%%%%%%%%%%%%%%%%%%%%%%%%%%%%%%%%%%%
%% Appendices
%%%%%%%%%%%%%%%%%%%%%%%%%%%%%%%%%%%%%%%%%%%%%%%%%%%%%%%%%%%%%%%%
\clearpage
\appendices
\input{appendices}

%%%%%%%%%%%%%%%%%%%%%%%%%%%%%%%%%%%%%%%%%%%%%%%%%%%%%%%%%%%%%%%%
%%%%%%%%%%%%%%%%%%%%%%%%%%%%%%%%%%%%%%%%%%%%%%%%%%%%%%%%%%%%%%%%
\end{document}

%% file: abstract.tex
%%%%%%%%%%%%%%%%%%%%%%%%%%%%%%%%%%%%%%%%%%%%%%%%%%%%%%%%%%%%%%
\begin{abstract}
%%%%%%%%%%%%%%%%%%%%%%%%%%%%%%%%%%%%%%%%%%%%%%%%%%%%%%%%%%%%%%
%%%%%%%%%%%%%%%%%%%%%%%%%%%%%%%%%%%%%%%%%%%%%%%%%%%%%%%%%%%%%%
% General Problem

This work studies rearrangement problems involving the sorting of
robots or objects in stack-like containers, which can be accessed only
from one side.  Two scenarios are considered: one where every robot or
object needs to reach a particular stack, and a setting in which each
robot has a distinct position within a stack. In both cases, the goal
is to minimize the number of stack removals that need to be performed.
Stack rearrangement is shown to be intimately connected to pebble
motion problems, a useful abstraction in multi-robot path planning. Through
this connection, feasibility of stack rearrangement can be readily
addressed. The paper continues to establish lower and upper bounds on
optimality, which differ only by a logarithmic factor, in terms of
stack removals. An algorithmic solution is then developed that
produces suboptimal paths much quicker than a pebble motion solver.
Furthermore, informed search-based methods are proposed for finding
high-quality solutions.  The efficiency and desirable scalability of
the methods is demonstrated in simulation.

%%%%%%%%%%%%%%%%%%%%%%%%%%%%%%%%%%%%%%%%%%%%%%%%%%%%%%%%%%%%%%
%%%%%%%%%%%%%%%%%%%%%%%%%%%%%%%%%%%%%%%%%%%%%%%%%%%%%%%%%%%%%%
\end{abstract}

%% file: intro.tex
%%%%%%%%%%%%%%%%%%%%%%%%%%%%%%%%%%%%%%%%%%%%%%%%%%%%%%%%%%%%%%
\section{Introduction}\label{sec:intro}
%%%%%%%%%%%%%%%%%%%%%%%%%%%%%%%%%%%%%%%%%%%%%%%%%%%%%%%%%%%%%%

Many robotic applications involve the handling of multiple stacks. For
instance, spatial restrictions faced by growing urban areas already
motivate stackable parking lots for vehicles, as in
Fig.~\ref{fig:gravity}(a).  With the advent of autonomous vehicles,
such solutions will become increasingly popular and will require
automation. Similarly, products in convenience stores are frequently
arranged in ``gravity flow'' shelving units depending on their type,
as in Fig.~\ref{fig:gravity}(b). Such stacks of products and materials
arise in the industry where a robot is able to interact with the foremost
object and perform operations similar to a ``push'' or a ``pop'' of a stack 
data structure.

%In both cases, the stack of robots or autonomous cars may need to be
%rearranged given a departure schedule or in order for a specific robot
%to exit the stack.

% Nick's website: https://www.westfaliaparking.com/news/article/automated-parking-growth/

% JJ's shipping container example
% A large scale example is the handling of stacks of containers 
% at ports (Fig.~\ref{fig:gravity}(b)). 

%Frequently, the objective in these domains is to arrange products of
%similar geometry on the same stack(s).

In the above stack rearrangement setups, the objective may be to
remove a specific object from the stack (e.g., a specific car from the
stackable parking) or to rearrange the objects into a specific
arrangement, which specifies the location of each object within a
stack (e.g., a Hanoi tower-like setting).  High quality solutions are
highly desirable for the applications, especially with regards to the
number of stack removals. Otherwise, an exorbitant amount of time is
spent performing redundant actions, which reduces efficiency or
appears unnatural to people.

Through a reduction to the pebble motions problem, which is
well-studied in the multi-robot literature, the feasibility of stack
rearrangement can be readily decided.  A naive feasible solution,
however, can be far from optimal in minimizing stack removals.
Adapting a divide-and-conquer technique, this paper establishes lower
and upper bounds on this number that differ by a mere logarithmic
factor.  Results are provided both for objects that need to be placed
in the right stack as well as the more general case where objects need
to acquire a specific position in the stack. Finally, the paper
considers both optimal and sub-optimal informed search methods and
proposes effective heuristics for stack rearrangement. This leads to
an experimental evaluation of the different algorithms and heuristics,
which suggests a combination that scales nicely with the number of
stackable objects.

\begin{figure}[tb!]
  \centering
  \vspace*{.5em}
  \begin{tabularx}{\linewidth}{@{}X@{}X@{}}
  \includegraphics[keepaspectratio, height=24.8mm]{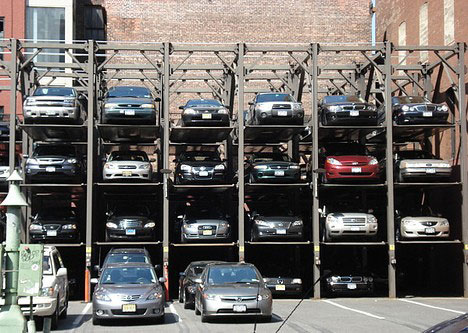} &
  \includegraphics[keepaspectratio, height=24.8mm]{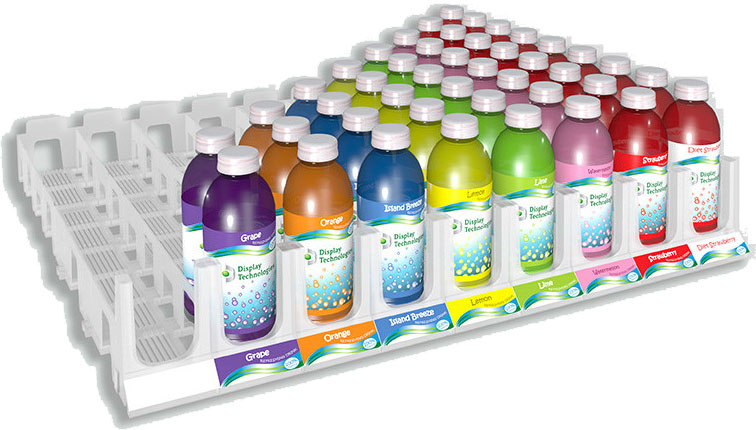}
  \end{tabularx}
  \vspace{0.01in}
  \caption{(a) Stackable parking lots are expected to become even more
  popular in urban environments with the advent of autonomous
  vehicles. (b) Rearranging stacks of objects is a task frequently
  encountered in convenience and grocery stores.}
\label{fig:gravity}
\end{figure}

%\begin{figure}[t]    
%    \centering
%    ~~\subfloat[Beverages]
%    {\includegraphics[keepaspectratio, height=26mm]{figures/shelf2.jpg}}\hfill
%    \subfloat[Score Four]
%    {\includegraphics[keepaspectratio, height=26mm]{figures/score4.jpg}}    
%    ~~
%    \caption{(a) A typical shelf-based environment found in convenience and
%            grocery stores. %
%            (b) The ``Score Four'' game , a 3-D version of the strategy
%            game Connect Four, where the object is to position four beads of the same color
%            in a straight line on any level or angle.}
%	\label{fig:gravity}
%\end{figure}

%% file: related-work.tex
%%%%%%%%%%%%%%%%%%%%%%%%%%%%%%%%%%%%%%%%%%%%%%%%%%%%%%%%%%%%%%
%%%%%%%%%%%%%%%%%%%%%%%%%%%%%%%%%%%%%%%%%%%%%%%%%%%%%%%%%%%%%%

%%%%%%%%%%%%%%%%%%%%%%%%%%%%%%%%%%%%%%%%%%%%%%%%%%%%%%%%%%%%%%
% \subsection{Multi-robot Path Planning (MPP)}
%%%%%%%%%%%%%%%%%%%%%%%%%%%%%%%%%%%%%%%%%%%%%%%%%%%%%%%%%%%%%%

{\bf Related Work:} \emph{Multi-body planning} is itself hard. In the
continuous case, complete approaches do not scale even though methods
try to decrease the effective DOFs \cite{AroBer+99}. For specific
geometries, e.g., unlabeled unit-discs among polygons, optimality is
possible \cite{SolYu+15}, even though the unlabeled case is still hard
\cite{SolHal15}. Given the problem's hardness, decoupled methods, such
as priority-based schemes \cite{BerOve05} or velocity tuning
\cite{LerLauSim99}, trade completeness for efficiency.

%Assembly planning \cite{WilLat94, HalLatWil00, SunRemAma01} deals with
%similar problems but few optimality arguments have been made.

Recent progress has been achieved for the discrete problem variant,
where robots occupy vertices and move along edges of a graph. For this
``pebble motion on a graph'' problem \cite{KorMilSpi84, CalDumPac08,
  AulMon+99, GorHas10}, feasibility can be answered in linear time and
paths can be acquired in polynomial time~\cite{KroLunBek13, LunBek11,
  WagKanCho12, YuLaV12}. The optimal variation is still hard but
recent optimal solvers with good practical efficiency have been
developed \cite{WagKanCho12, YuLaV12, SharSte+15, YuLaV16}. The
current work is motivated by this progress and aims to show that for
stack rearrangement it is possible to come up with practically
efficient algorithms.

%%%%%%%%%%%%%%%%%%%%%%%%%%%%%%%%%%%%%%%%%%%%%%%%%%%%%%%%%%%%%%
% \subsection{Automation and General Rearrangement}
%%%%%%%%%%%%%%%%%%%%%%%%%%%%%%%%%%%%%%%%%%%%%%%%%%%%%%%%%%%%%%

General rearrangement planning \cite{BenRiv98, Ota04} is also hard,
similar to the related ``navigation among movable obstacles'' ({\tt
NAMO}) \cite{Wil91, CheHwa91, DemORoDem00, NieStaOve06, BerSti+08},
which can be extended to manipulation among movable obstacles ({\tt
MAMO}) and related challenges \cite{StiSch+07, HavOzb+14, SriFan+14,
GarLozKae14, KroSho+14, KroBek15a, KroBek16}.  These efforts focus on
feasibility and no solution quality arguments have been provided.
A recent work has focused, similar to
the current paper, on high-quality rearrangement solutions but in the
context of manipulation challenges in tabletop
environments \cite{HanSti+17}.

%Asymptotic optimality has been achieved for ``minimum constraint
%removal'' paths \cite{Hau14, Hau13}, which, however, do not consider
%negative object interactions.

%Rearrangement planning can be seen as an instance of integrated task
%and motion planning~\cite{HauNgT11, KaeLoz12}. While many integrated
%task and motion planners exist, they frequently do not provide
%optimality guarantees \cite{CamAlaGra09, PlaHag10, SriFan+14,
%  GarLozKae14, GhaLalAla15, DanKinCha+16}.  Recent work on
%asymptotically optimal task planning is at this point prohibitively
%expensive for practical use \cite{VegRoy16}.

%This work does not deal with other aspects of rearrangement, such as
%arm motion \cite{SimLau+04, BerSriKuf12, CohChiLik13, ZucRat+13} or
%grasp planning \cite{CioAll09, BohMor+14}.  Non-prehensile actions,
%such as pushing, are also not considered \cite{CosHer+11,
%DogSri11}.

%
%Manipulation problems have been approached via multi-modal roadmaps to 
%deal both with discrete and continuous parameters~\cite{BreLal+04, HauNgT11}. 
%These methods emphasize the key insight that methods should 
%reason over the properties of states without enumerating them. 

%% file: problem.tex
%%%%%%%%%%%%%%%%%%%%%%%%%%%%%%%%%%%%%%%%%%%%%%%%%%%%%%%%%%%%%%
\section{Problem Formulation}\label{sec:problem}
%%%%%%%%%%%%%%%%%%%%%%%%%%%%%%%%%%%%%%%%%%%%%%%%%%%%%%%%%%%%%%
%%%%%%%%%%%%%%%%%%%%%%%%%%%%%%%%%%%%%%%%%%%%%%%%%%%%%%%%%%%%%%

% Suppose there are $n$ objects $\objects = \{o_1, \ldots, o_n\}$
% placed on $w + 1$ columns. Each column in turn 
% models a last-in-first-out (LIFO) queue\footnote{Colloquially referred 
% to as a ``stack'' in Computer Science. The key functionality is that 
% elements can only be added or removed from one end (commonly referred 
% to as the ``top'') of the data structure.} of fixed uniform capacity.
% We assume $w \ge 2$~\footnote{Robotic rearrangement with a total of two stacks is 
% impossible without additional swap space.}
% and each column has an integer {\em depth} of $d \ge 1$, 
% which means it is capable of storing up to $d$ objects.
% Throughout the paper, we assume $n \le wd$, unless otherwise specified. 

Assume $n$ objects $\objects = \{o_1, \ldots, o_n\}$ that occupy $w +
1$ last-in-first-out (LIFO) queues, i.e., {\em stacks}, where $w \geq
2$, since 2-stack rearrangement is impossible. Elements can only be
added or removed from one end of the data structure, often referred to
as the ``top''.  Furthermore, each stack has an integer {\em depth} of
$d \ge 1$, corresponding to the maximum stack capacity.  An object at
the top of a stack has a depth of $1$.  The assumption is that $n \le
wd$, unless specified otherwise.

% Since each column is modeled as a LIFO queue, at most a single 
% object from a column is accessible at a time. We denote the end of a column
% that is accessible as the {\em top} of the column.

%The exact behavior that gravity has on objects within a stack, either
%pulling objects toward the top, or pulling objects to the bottom, is
%inconsequential.

Modeling many real world problems, the assumption is that objects in a
stack always occupy contiguous positions, e.g., if the top object is
removed from a stack in Fig.~\ref{fig:gravity}(b), the remaining
objects in the stack will ``slide'' to the front. Similarly, as an
object is pushed onto a stack, the existing objects will shift
backwards by one position.  It is straightforward to see that the two
versions of the problem, i.e., a top-down or a bottom-up stack, as
shown in Fig.~\ref{fig:abstract-gravity}, induce the same problem.
For consistency, the bottom-up setting is used for the remainder of
the paper.

\begin{figure}[ht]
  \vspace{0.02in}
    \centering
    \begin{tabular}{ccc}
        \begin{tikzpicture}
            \foreach \nodeName/\nodeLocation in {
                1/{(0.5,1.5)},2/{(0.5,1)},3/{(0.5,0.5)},4/{(1.5,1)},5/{(1.5,0.5)},6/{(2.5,0.5)}}{
                \node[thick,shape=rectangle,draw=black,fill=node_fill,text=black,minimum width=0.8cm,minimum height=0.4cm] 
                (\nodeName) at \nodeLocation {\small $\nodeName$}; 
            }
            \draw[-,thick] (0,0.2) to (0,1.8);
            \draw[-,thick] (0,1.8) to (3,1.8);
            \draw[-,thick] (3,1.8) to (3,0.2);
            \draw[-,thick] (2,1.8) to (2,0.2);
            \draw[-,thick] (1,1.8) to (1,0.2);
        \end{tikzpicture}
        &&
        \begin{tikzpicture}
            \foreach \nodeName/\nodeLocation in {
                1/{(0.5,1.5)},2/{(0.5,1)},3/{(0.5,0.5)},4/{(1.5,1.5)},5/{(1.5,1)},6/{(2.5,1.5)}}{
                \node[thick,shape=rectangle,draw=black,fill=node_fill,text=black,minimum width=0.8cm,minimum height=0.4cm] 
                (\nodeName) at \nodeLocation {\small $\nodeName$}; 
            }
            \draw[-,thick] (0,0.2) to (0,1.8);
            \draw[-,thick] (0,1.8) to (3,1.8);
            \draw[-,thick] (3,1.8) to (3,0.2);
            \draw[-,thick] (2,1.8) to (2,0.2);
            \draw[-,thick] (1,1.8) to (1,0.2);
        \end{tikzpicture}\\
        {\footnotesize (a) Top-down} && {\footnotesize (b) Bottom-up} \\
    \end{tabular}
  \vspace{0.05in}
    \caption{Visualization of the abstract problem where objects 
            (a) slide to the front of the stacks,
            (b) are pulled to the bottom of the stack.
    }
    \label{fig:abstract-gravity}
\end{figure}
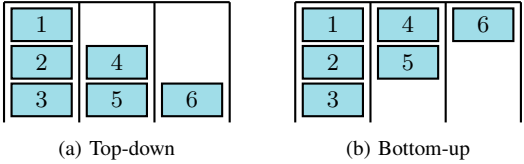

% Pick and place definition

Using this setup, an object that currently resides at the top of stack
$i$ can be transferred to an arbitrary stack $j$ via a {\em \tpap
  action}, denoted as $a_{ij}$.  A {\em permissible \tpap action}
constrains the above definition by requiring that $i$ be non-empty,
and that $j$ not currently be at capacity.

An {\em arrangement} is an injective mapping $\pi: \mathcal O \to
\mathbb N^2$, $o_i \mapsto \pi(o_i)$.  Here, $\pi(o_i)$ is a 2-tuple
$(\pi^1(o_i), \pi^2(o_i))$ in which $1 \le \pi^1(o_i) \le w$ and $1
\le \pi^2(o_i) \le d$ are the stack and depth locations of $o_i$,
respectively.  The paper primarily focuses on two main problems,
defined as follows.

\begin{problem}{\bf Labeled Stack Rearrangement (\lsr).}
	Given $\langle \objects, w, d, \arrangement_I, \arrangement_G
        \rangle$, compute a sequence of \tpap actions $A =
        (a_{i_1j_1}, a_{i_2j_2},\dots)$ that move the objects one at a
        time from an initial $\arrangement_I$ to a goal arrangement
        $\arrangement_G$.
\end{problem}

{\hbadness=10000
\begin{problem}{\bf Column-Labeled Stack Rearrangement (\clsr).}
    Similar to \lsr, but the objects are only required to be moved to their
    goal stacks without a specific depth. That is, $\pi_G^{2}$ is 
	left unspecified for all objects. 
\end{problem}
}

%So in \sorc, $\arrangement_g$ can be replaced by $\arrangement_*$ so that
%$\forall o_i \in \objects, \arrangement_g^{(1)} (o_i) = \arrangement_*^{(1)}
%(o_i)$, where $\arrangement^{(1)}$ denotes the first element in tuple
%$\arrangement$.

Whereas \clsr can be viewed as a sub-problem in approaching \lsr, it
has practical incarnations - perhaps even more so than \lsr. For
example, in retail, it is almost always the case that a shelf slot
holds the same type of product (e.g.,
Fig.~\ref{fig:gravity}(b)). Solving \clsr then corresponds to
rearranging an out of order shelf so that each stack holds only a
single type of product.

In this paper, the {\em optimization objective} is to minimize the
number of actions taken, i.e. $|A|$.  For robotic manipulation, the
objective models the required number of grasps by the robotic
manipulator, which is frequently the limiting factor.

%% file: feasibility.tex
\section{Structural Analysis}\label{sec:structure}
%%%%%%%%%%%%%%%%%%%%%%%%%%%%%%%%%%%%%%%%%%%%%%%%%%%%%%%%%%%%%%
% \subsection{Relation to Pebble Motion Problems}
%%%%%%%%%%%%%%%%%%%%%%%%%%%%%%%%%%%%%%%%%%%%%%%%%%%%%%%%%%%%%%

% A closely related problem is that of {\em Pebble Motion on Graphs}

% A \pmg instance is a tuple $\langle G, x_I, x_G \rangle$, in which
% $x_I$ and $x_G$ are the initial and goal {\em configurations} of the
% pebbles. The goal of \pmg is to decide $x_G$ is reachable from
% $x_I$, and to subsequently find a sequence of moves to reach $x_I$
% when possible.

A closely related problem is {\em Pebble Motion on Graphs}
(\pmg)~\cite{KorMilSpi84}: suppose an undirected graph $G = (V, E)$
has $p < |V|$ pebbles placed on distinct vertices and which can move
sequentially to adjacent empty vertices.  Given a \pmg\ instance
$\langle G, x_I, x_G \rangle$, the goal of \pmg is to decide if the
{\em configuration} $x_G$ is reachable from $x_I$, and to subsequently
find a sequence of moves to do so when possible.  When $G$ is a tree,
this problem is referred to as {\em Pebble Motion on Trees} (\pmt).
The considered versions of \lsr (and \clsr) are \pmt problems.

\begin{figure}[ht]
    \centering
    \begin{tabular}{ccc}
        \begin{tikzpicture}
            \foreach \nodeName/\nodeLocation in {
                1/{(0.5,1.5)},2/{(0.5,1)},3/{(0.5,0.5)},4/{(1.5,1.5)},5/{(1.5,1)},6/{(2.5,1.5)}}{
                \node[thick,shape=rectangle,draw=black,fill=node_fill,text=black,minimum width=0.8cm,minimum height=0.4cm] 
                (\nodeName) at \nodeLocation {\small $\nodeName$}; 
            }
            \draw[-,thick] (0,0.2) to (0,1.8);
            \draw[-,thick] (0,1.8) to (3,1.8);
            \draw[-,thick] (3,1.8) to (3,0.2);
            \draw[-,thick] (2,1.8) to (2,0.2);
            \draw[-,thick] (1,1.8) to (1,0.2);
        \end{tikzpicture}
        &&
        \begin{tikzpicture}
            \foreach \nodeLocation in {
                {(0.5,0.5)},{(1.5,1.5)},{(2.5,0.5)},
                {(0.5,1.5)},{(1.5,0.5)},{(2.5,1.5)},
                {(0.5,1)},{(1.5,1)},{(2.5,1)},{(1.5,0)}}{
                \node[thick,shape=circle,draw=black,fill=black,minimum size=0.15cm,inner sep=0] at \nodeLocation {}; 
            }
            \draw[-,thick] (0.5,0.5) to (0.5,1.5);
            \draw[-,thick] (1.5,0.5) to (1.5,1.5);
            \draw[-,thick] (2.5,0.5) to (2.5,1.5);
            \draw[-,thick] (0.5,0.5) to (1.5,0);
            \draw[-,thick] (1.5,0.5) to (1.5,0);
            \draw[-,thick] (2.5,0.5) to (1.5,0);
            \foreach \nodeName/\nodeLocation in {
                1/{(0.5,1.5)},2/{(0.5,1)},3/{(0.5,0.5)},4/{(1.5,1.5)},5/{(1.5,1)},6/{(2.5,1.5)}}{
                \node[thick,shape=circle,draw=black,fill=node_fill,text=black,minimum size=0.4cm,inner sep=0.05cm] 
                (\nodeName) at \nodeLocation {\small $\nodeName$}; 
            }
        \end{tikzpicture}\\
        {\footnotesize (a) An \lsr instance} && {\footnotesize (b) A
          \pmt instance}
        \vspace{0.05in}
    \end{tabular}
    \caption{From \lsr to \pmt}
    \label{fig:sor-to-pmt}
    \vspace{-.1in}
\end{figure}

\begin{proposition}\label{p:sor-to-pmt} An \lsr instance is always 
reducible to a \pmt instance. In particular, a solution to the 
reduced \pmt instance is also a solution to the initial \lsr instance. 
\end{proposition}
\vspace{-0.15in}
\begin{proof}
% Assume the \lsr formulation where gravity pulls objects toward the
% bottom of the stack columns.  
% the manipulator can be represented by a root of a tree in constructing the
% corresponding \pmt instance. That is, the manipulator {\em connects}
% all the column entrances as in Fig.~\ref{fig:sor-to-pmt}. 
% Viewing each column as a path of length $d$ and the manipulator as a
  % vertex,

Given an \lsr instance $\langle \objects, w, d, \arrangement_I,
\arrangement_G \rangle$, as shown in Fig.~\ref{fig:sor-to-pmt}, the
tree graph $T = (V, E)$ in the \pmt\ instance is obtained by first
viewing each stack as a path of length $d$, and then joining the top
vertices of these stacks with a root vertex, which builds the
connection between them. This yields $|V| = ((w + 1)d) + 1$ vertices.
It is clear that object arrangements $\pi_I$ and $\pi_G$ directly map
to configurations $x_I$ and $x_G$ of a \pmt instance. Note that a
\tpap action in the \lsr solution is equivalent to moving one pebble
from a path on $T$ to another path through the root vertex.
Similarly, given a solution to the \pmt instance, a solution to the
\lsr instance can be constructed by treating a pebble passing through
the root as a \tpap action.
\end{proof}

\vspace{-0.1in}
Given the relationship between \pmt and \lsr, and that finding optimal
solutions (i.e., a shortest solution sequence) for \pmg and many of
its variants is NP-hard \cite{Gol11,RatWar90}, there is evidence to
believe that optimally solving \lsr (i.e., minimizing the number of
actions) is also NP-hard.

% \subsection{Feasibility of \lsr}

In terms of feasibility, the \lsr\ problem is always feasible as
defined. This is due to the assumption that $n \leq wd$ while the
total number of slots in the stacks are $(w+1)d$. This allows to
always clear one stack of depth $d$ and then the elements in the
remaining stacks can be arranged with the aid of the empty one.
Consider, however, a more general version, called \glsr, that allows
for $n$ to exceed $wd$.  So there may be fewer than $d$ buffers
available to rearrange objects.

Note that Proposition~\ref{p:sor-to-pmt} still holds for \glsr. The
mapping from \glsr to \pmt immediately leads to algorithmic solutions
for \glsr (and therefore, \lsr). By Proposition~\ref{p:sor-to-pmt}, a
\glsr instance is feasible if and only if the corresponding \pmt
instance is so. The feasibility test of \pmt can be performed in
linear time \cite{AulMon+99}, so the same is true for \glsr as the
reduction can be performed also in linear time.

For a feasible \glsr, solving the corresponding \pmt can be performed
in $O(|V|^3)$ running time (and pebble moves) \cite{KorMilSpi84}. This
translates to a solution for \glsr that runs in $O(w^3d^3)$ time using
up to $O(w^3d^3)$ actions.  This result, however, does not extend to
\lsr (i.e., when $n \le wd$) as an \lsr instance is in fact always
feasible. \emph{It turns out that \lsr can be solved in less
  computation time and a reduced number of actions than using a \pmt
  solver}. This contribution is established in the following
proposition.

\vspace{-.05in}
\begin{proposition}\label{p:trivial-ub}An arbitrary \lsr can be solved using 
$O(wd^2)$ \tpap actions. 
\end{proposition}
\vspace{-0.15in}
\begin{proof}
Consider an \lsr with $n = wd$. Without loss of generality, assume
that: $\forall\ o \in \mathcal O: \pi_{I}^{1}(o) \leq w,
\pi_{G}^{1}(o) \leq w,$ i.e., stack $w + 1$ is empty at the start and
goal state.  It suffices to show that one stack (e.g., the first) can
be arranged in $O(d^2)$ actions and this can be repeated $w$ times.
The $O(d^2)$ cost for a stack is because each object can be moved to
its destination in $O(d)$ moves and this can be repeated for $d$
times. Consider the object $o$ to be moved at the bottom of the first
stack, i.e., $\pi_G(o) = (1, d)$. Without loss of generality, assume
that $\pi_I(o) = (x, y)$ with $x \neq 1$. Initially, $o$ will be moved
to the top of stack $x$. If $y = 1$, no action is needed. Otherwise,
perform the following moves per Fig.~\ref{fig:feasibility-1}: {\em
  (i)} move the object at $(1, 1)$ to the buffer stack $(w+1)$, {\em
  (ii)} move objects from $(x, 1)$ to $(x, y - 1)$ to the buffer, {\em
  (iii)} move $o$ to $(1, 1)$, {\em (iv)} move all objects in the
buffer except the last to stack $x$, {\em (v)} move $o$ to the top of
stack $x$, and {\em (vi)} move the last object in the buffer to stack
$1$.
\begin{figure}[ht]
    \centering
    \includegraphics[width=\linewidth]{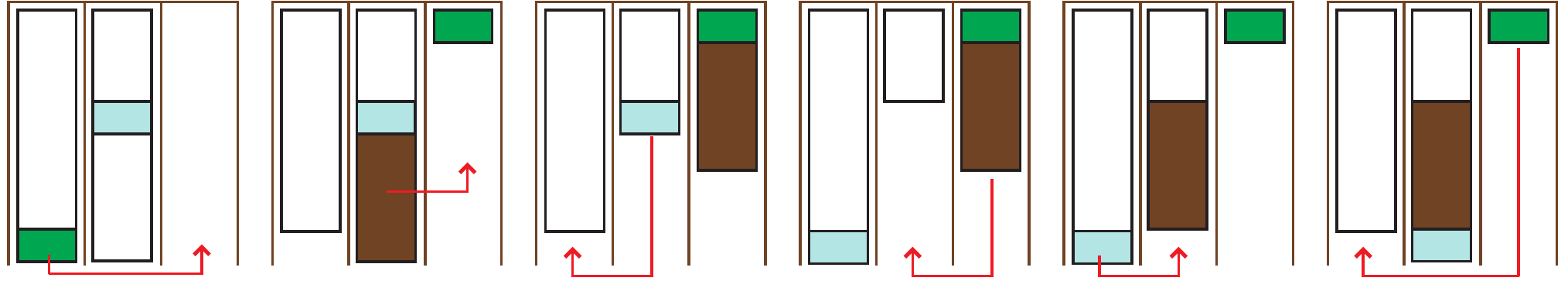}
    \caption{The cyan object moves to the top of its stack $x$
		with $O(d)$ actions.}
    \label{fig:feasibility-1}
\end{figure}

Using the same $O(d)$ procedure, the object $o'$ at $(1, d)$ can be
moved to $(1, 1)$. Using the buffer stack $(w+1)$, $o$ and $o'$ can be
swapped in three actions. Then, reverting the sequences, $o$ can be
moved to $(1, d)$ for an $O(d)$ total number of actions . So,
arranging a single stack needs $O(d^2)$ actions and the entire problem
takes $O(wd^2)$ actions.
\end{proof}
\vspace{-0.1in}

The running time is also bounded by $O(wd^2)$ since the only
computation cost is to go through $\pi_I$ and $\pi_G$ and recover the
solution sequence. For \glsr, which allows $n > wd$, an arbitrary
instance may not be feasible. This paper focuses on the optimal number
of actions for rearrangement problems, so \glsr is not considered
further.

%% file: bounds.tex
%%%%%%%%%%%%%%%%%%%%%%%%%%%%%%%%%%%%%%%%%%%%%%%%%%%%%%%%%%%%%%
\section{Fundamental Bounds on Optimality}\label{sec:bounds}
%%%%%%%%%%%%%%%%%%%%%%%%%%%%%%%%%%%%%%%%%%%%%%%%%%%%%%%%%%%%%%
%%%%%%%%%%%%%%%%%%%%%%%%%%%%%%%%%%%%%%%%%%%%%%%%%%%%%%%%%%%%%%

This section provides an analysis on the structure properties of \lsr,
focusing on the fundamental optimality bounds and polynomial time
algorithms for computing them.  The analysis assumes the hardest case
of \lsr where $n$ equals $wd$.  Without loss of generality, it is
assumed that stack $(w+1)$ is empty at the initial and goal state,
serving as a {\em buffer}. First, consider the lower bound on the
number of required actions.

%%%%%%%%%%%%%%%%%%%%%%%%%%%%%%%%%%%%%%%%%%%%%%%%%%%%%%%%%%%%%%
% \subsection{Lower Bounds}
%%%%%%%%%%%%%%%%%%%%%%%%%%%%%%%%%%%%%%%%%%%%%%%%%%%%%%%%%%%%%%

\vspace{-0.05in}
\begin{proposition}\label{p:average-lb}In the average case, $\Omega(wd)$ 
actions are required for solving \lsr. 
\end{proposition}
\vspace{-0.15in}
\begin{proof}

First consider a worst case scenario, i.e., that the deepest objects
$o_i \in \mathcal O$ in each stack $k$, i.e., $\pi_I(o_i) = (k, d)$,
must be moved to the next stack $k+1$ modulo $w$, i.e., $\pi_G(o_i) =
((k+1) \mod\ w, d)$. To move each of these objects, at least $d$
actions are needed because $d - 1$ objects are blocking the way to
them. Therefore, the total number of required actions is $\Omega(wd)$.

In the average case (assuming $\pi_I$ and $\pi_G$ are both uniformly
randomly generated), for each stack $k$, $1 \le k \le w$, look at an
object along the stack. Each object $o$ has probability $(w-1)/w$ to
have $\pi_I^{1}(o) \ne \pi_G^{1}(o)$.  That is, with probability
$1/w$, $o$ will stay in stack $k$ and with probability $(w-1)/w$ it
must be moved to a different stack. Because moving $o$ will require on
average $d/2$ actions, the expected cost of moving it is then
$(w-1)d/(2w)$. For all $w$ stacks, this is then $\Omega((w-1)wd/(2w))
= \Omega(wd)$.
\end{proof}
\vspace{-0.1in}

For the case of $w \ll d$, better average case lower bounds can be
found per the following lemma.

\vspace{-.05in}
\begin{lemma}\label{l:lb}The number of moves for solving \lsr is 
$\Omega(wd\log d/ \log w)$. 
\end{lemma}
  \vspace{-.15in}
\begin{proof}
The bound is established by counting the possible \lsr problems for
fixed $w$ and $d$, i.e., for the case $n = wd$.  Given a fixed
$\pi_I$, there are $n!$ possible $\pi_G$, so there are at least $n! =
(wd)!$ different \lsr instances. With each action, one object at the
top of a stack ($w + 1$ of these) can be moved to any other stack ($w$
of these). Therefore, each action can create at most $w(w + 1) <
(w+1)^2$ new configurations. In order to solve all possible \lsr
instances, it must then be the case that the required number of moves,
defined as $\min \{|A|\}$, must satisfy $[(w + 1)^2]^{\min \{|A|\}}
\ge (wd)! $.  Then, by Sterling's approximation, $\min \{|A|\} = \Omega(wd
\frac{\log d}{\log w})$.
\end{proof}
\vspace{-0.05in}

%~\cite{Sti30} - KB: we can save some space by not referencing
%Sterling's approximation, bring back for JOURNAL version

Lemma~\ref{l:lb} implies Proposition~\ref{p:average-lb} as well but
does so in a less direct way. Interestingly, Lemma~\ref{l:lb}
immediately implies the following better lower bounds.

\vspace{-0.05in}
\begin{corollary}For an \lsr with $w = e^{\sqrt{\log d}}$, on average it 
requires $\Omega(wd \sqrt{\log d})$ actions to solve. 
\end{corollary}
\vspace{-0.15in}

\begin{corollary}\label{c:const-w-lower}For an \lsr with $w$ being a 
constant, on average it requires $\Omega(d \log d)$ actions to solve. 
\end{corollary}
\vspace{-0.05in}

The focus now shifts towards upper bounds on optimality where
polynomial time algorithms are presented for computing them.  Recall
that a trivial upper bound of $O(wd^2)$ is given by
Proposition~\ref{p:trivial-ub}.  Comparing the $O(wd^2)$ upper bound
with the lower bound, which ranges between $\Omega(wd)$ and
$\Omega(d\log d)$ (for constant $w$), there remains a sizable gap.
Forthcoming algorithms illustrate how to significantly reduce, and in
certain cases eliminate this gap.

% \TODO{[} First, we define
% a simpler version of \lsr where the task is to partially solve an \lsr 
% instance so that each object is sorted into the column it should belong 
% to. Denote this version of the \lsr as \clsr. Essentially, \clsr is a 
% specialized \lsr such that for an object $o \in \mathcal O$, $\pi_I(o)$
% is fully specified but $\pi_G(o)$ only has the $\pi_G^{1}(o)$ specified
% whereas $\pi_G^{2}(o)$ is arbitrary. Alternatively, \clsr can be viewed
% as sorting $wd$ objects with $w$ colors and each color is shared by $d$ 
% objects. The goal is then to sort the objects so that each column has a 
% single color. \TODO{], \clsr\ already defined, 
% delete this part of merge it to the problem formulation.} 
% Interestingly, \clsr can be solved using $O(wd \log w)$
% actions. 

\vspace{-0.05in}
\begin{lemma}\label{l:column-sorting}An arbitrary instance of \clsr can be 
solved using $O(wd\log w)$ actions. 
\end{lemma}
\vspace{-0.15in}
\begin{proof}
%As one might have guessed, 
The $\log w$ factor in $O(wd \log w)$ comes from a divide-and-conquer
approach. As such, a recursive algorithm is outlined for solving
\clsr. In the first iteration, partition all $wd$ objects into two
sets based on $\pi_G^{1}$. For an object $o \in \mathcal O$, if
$\pi_G^{1}(o) \le \lceil w/2 \rceil$, then it is assigned to the {\em
  left} set. Otherwise it is assigned to {\em right} set. The goal of
the first iteration is to sort objects so that the left set resides in
stacks $1$ to $\lceil w/2 \rceil$, as illustrated in
Fig.~\ref{fig:left-right}.

\begin{figure}[th]
    \centering
    \includegraphics[width=.8\linewidth]{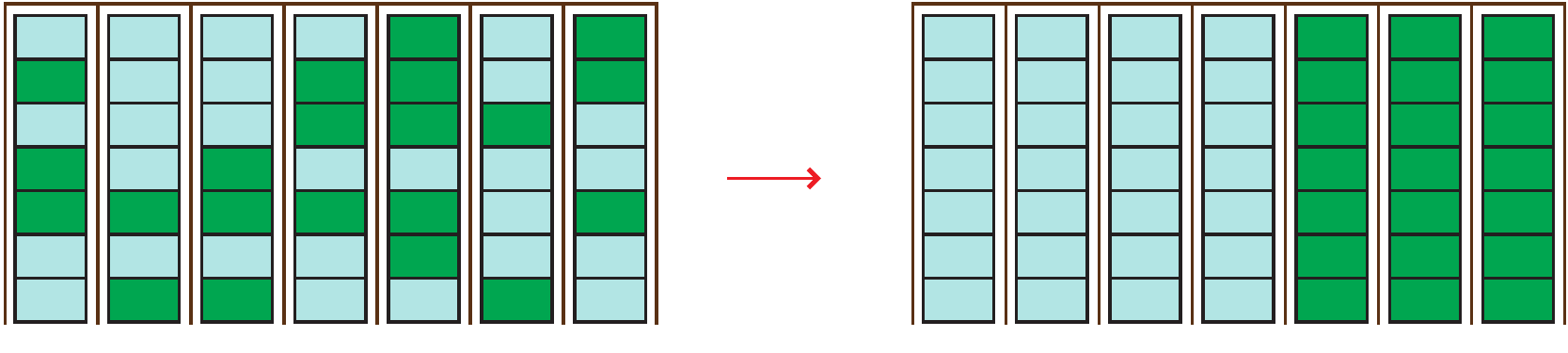}
    \caption{The goal of the first iteration in solving an \clsr instance
		with $w = 7$ and $d = 7$. The empty stack is not drawn.}
    \label{fig:left-right}
\end{figure}

In the first iteration, begin with the first stack and sort it into two
contiguous sections belonging to the left set and the right set.  The process
involves using another occupied stack and the buffer stack (using $O(d)$ moves). 
Note that the content of the other occupied stack is irrelevant. These three 
stacks are illustrated in the first figure in Fig.~\ref{fig:sorc-1}. Assume 
that $\ell$ objects of stack $1$ belong to the left set (in the example, 
$\ell = 4$). To begin, the top $\ell$ objects of the last stack is moved to 
the buffer. This allows the sorting of the first stack into two contiguous 
blocks of left only and right only objects, which can then be returned to the 
first stack. Note that the order of the two blocks can be reversed using the 
same procedure; this will be used shortly. The procedure is then applied to 
all stacks. The procedure and the end result are illustrated in 
Fig.~\ref{fig:sorc-1}. It is clear that the total actions required is $O(wd)$. 

\vspace{0.05in}
\begin{figure}[ht]
    \centering
    \includegraphics[width=\linewidth]{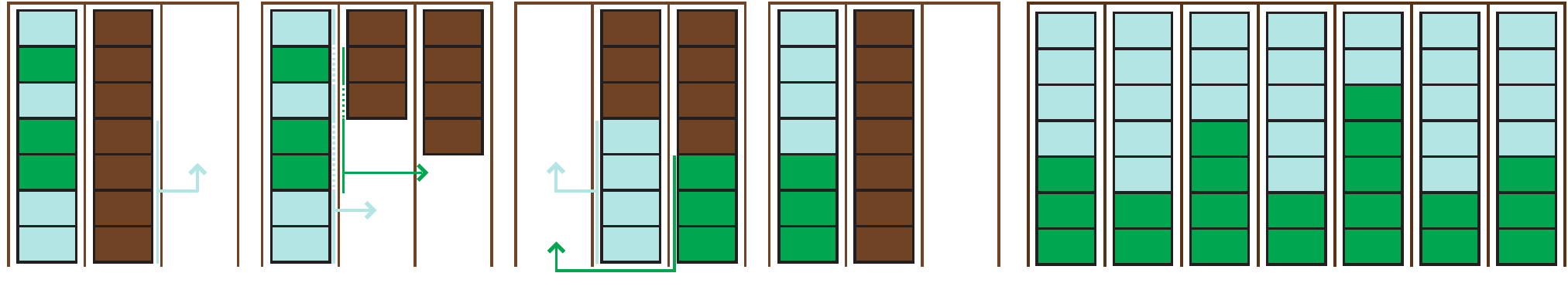}
    \caption{The left four figures illustrate the process of sorting a 
		single stack into two contiguous blocks. The last figure is the 
		end result of applying the procedure to all stacks.}
    \label{fig:sorc-1}
\end{figure}

The next step involves the first two stacks and attempts to {\em consolidate} 
the sets. If any stack is already fully occupied by either the left or
the right set, then that stack can be skipped; suppose not. Let these 
two stacks be $i$-th and $j$-th stacks and let $\ell_i$ and $\ell_j$ 
be the number of objects belonging to the left set in the $i$-th and 
$j$-th stacks, respectively. If $\ell_i + \ell_j \ge d$, then using the
buffer stack, stack $i$ can be forced to contain only objects belonging to 
the left set. Fig.~\ref{fig:sorc-2} illustrates applying the procedure 
to the left most two stacks to the running example and the result. 

\vspace{0.05in}
\begin{figure}[ht]
    \centering
    \includegraphics[width=\linewidth]{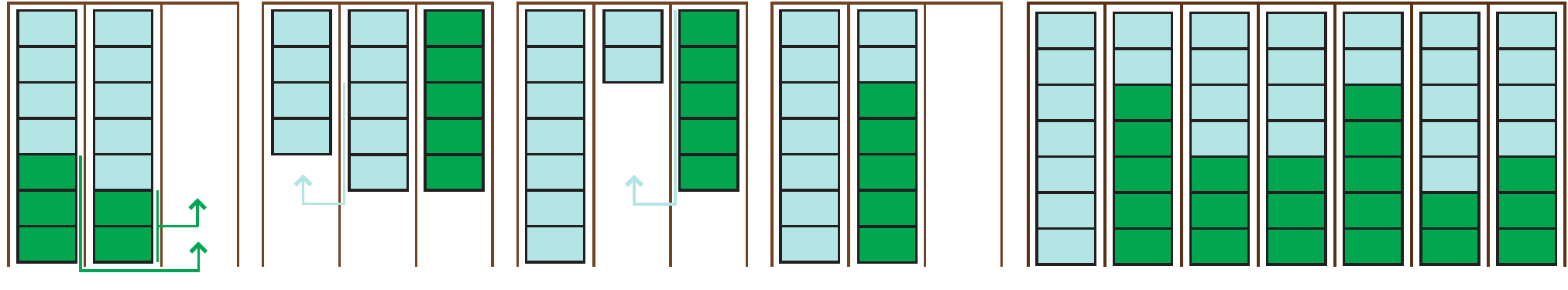}
    \caption{Consolidating the first two stacks.}
    \label{fig:sorc-2}
\end{figure}

If $\ell_i + \ell_j < d$, then stack $i$ is processed so that the
$\ell_i$ objects belonging to the left set are on the top (using the
block reverse procedure mentioned earlier in this proof). Then, a
similar consolidation routine can be applied. Fig.~\ref{fig:sorc-3}
illustrates the application of the procedure to stacks $2$ and $3$ of
the right most figure of Fig.~\ref{fig:sorc-2}. With these two
variations, all stacks can be sorted so that each stack contains only
objects from either the left set or the right set.

\vspace{0.05in}
\begin{figure}[ht]
    \centering
    \includegraphics[width=\linewidth]{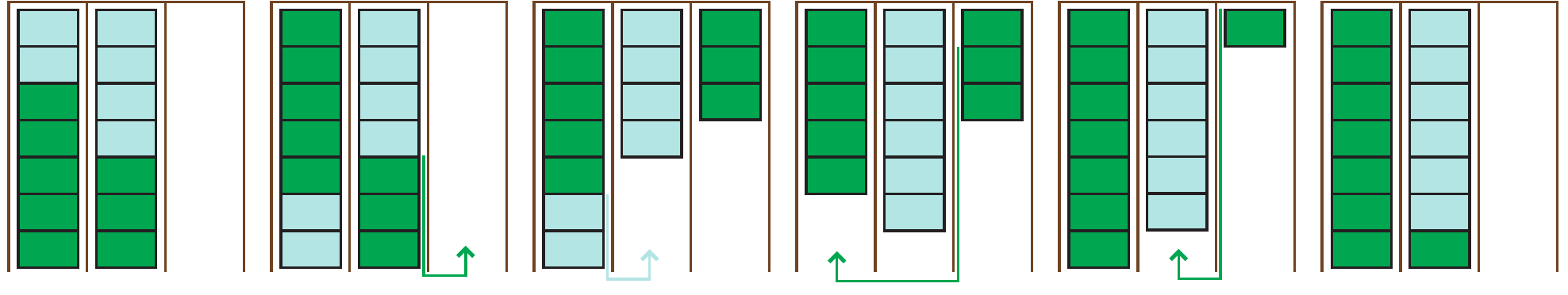}
    \caption{Consolidating the first two stacks that requires reversing
		the two contiguous blocks one of the two stacks.}
    \label{fig:sorc-3}
\end{figure}

At this point, using the buffer stack, entire stacks can be readily swapped to
complete the first iteration. The total number of actions used is $O(wd)$ per
iteration. Applying the same iterative procedure to the left and the right sets
of objects, the full \clsr problem can then be solved with $O(wd\log w)$
actions.
\end{proof}
\vspace{-0.1in}

After solving \clsr, each stack needs to be resorted to fully solve 
the original \lsr problem, which can be performed using $O(d\log d)$ actions.

\vspace{-0.05in}
\begin{lemma}\label{l:small-w}After solving the \clsr portion of an 
\lsr instance, a stack can be fully sorted using another stack and
the buffer stack with $O(d\log d)$ actions. 
\end{lemma}
\vspace{-0.15in}
\begin{proof}The sorting is done recursively. Suppose stack 
$i$ is to be sorted using stack $j$ and the buffer stack. Assume without loss of generality 
that $d = 2^k$ for some $k$. To start, move half of the objects in stack 
$j$ to the buffer stack. This creates two buffers of size $2^{k-1}$. Using
these two buffers, stack $i$ can be sorted into a top half and a bottom half. 
As these two halves are restored to stack $i$, the top and bottom halves are 
separated. Iteratively applying the same procedure can then sort stack $i$ 
fully in $\log d$ iterations. The total number of required actions is then
$O(d\log d)$. Fig.~\ref{fig:sor} provides an illustrative example sorting 
sequence for $k =3$. 
%For example, consider the case of $d = 8 = 2^3$, as depicted in Fig.~\ref{fig:sor}a. 
%After splitting the second stack into two halves (Fig.~\ref{fig:sor}b), 
%the first stack can be sorted so that objects 
%$1$-$4$ and objects $5$-$8$ are separated (Fig.~\ref{fig:sor}c). 
%After restoring these halves (Fig.~\ref{fig:sor}d), 
%the top half (objects $1$-$4$) can be further
%sorted similarly (Fig.~\ref{fig:sor}e); same 
%applies to objects $5$-$8$ (Fig.~\ref{fig:sor}f). 
%Repeating the procedure one more iteration can then sort the entire stack.
\end{proof}

\begin{figure}[ht]
    \centering
    \begin{overpic}[width=\linewidth]{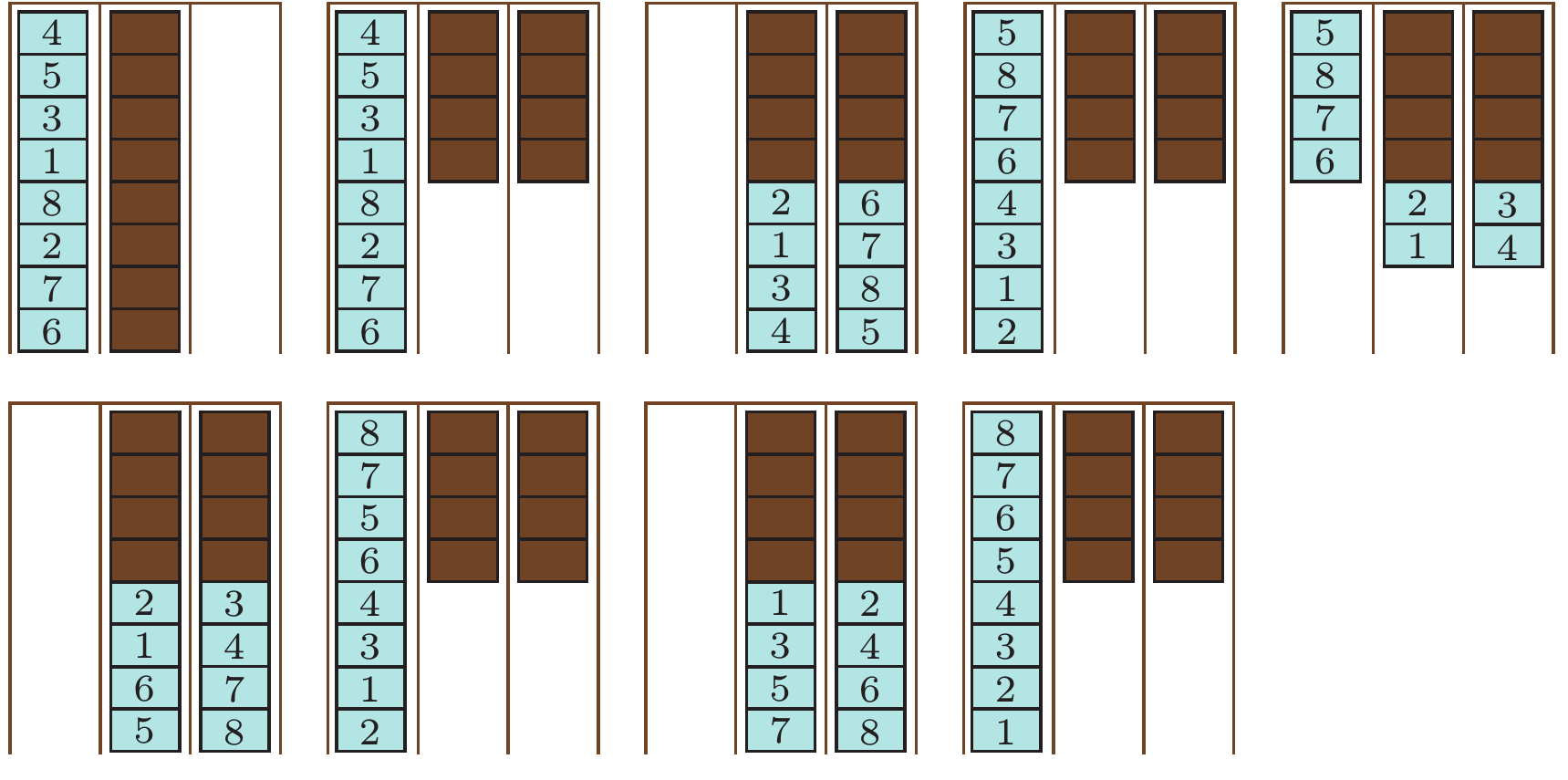}
        \put(7.5,24.75){\footnotesize (a)}
        \put(27.75,24.75){\footnotesize (b)}
        \put(48.25,24.75){\footnotesize (c)}
        \put(68.25,24.75){\footnotesize (d)}
        \put(88.75,24.75){\footnotesize (e)}
        \put(7.5,-1){\footnotesize (f)}
        \put(27.75,-1){\footnotesize (g)}
        \put(48,-1){\footnotesize (h)}
        \put(68.5,-1){\footnotesize (i)}
    \end{overpic}
    \vspace{0.01in}
    \caption{An example of sorting a stack with $d = 2^3$.}
    \label{fig:sor}
\end{figure}
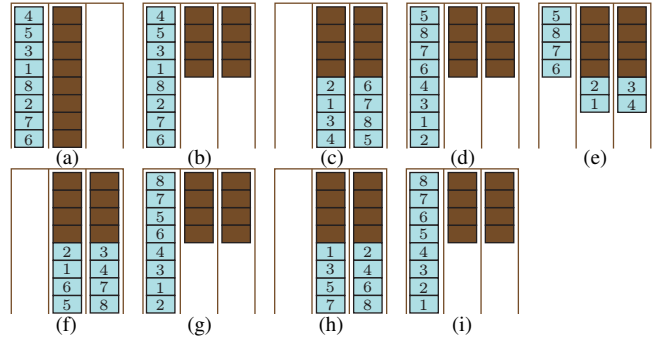

Lemma~\ref{l:column-sorting} and Lemma~\ref{l:small-w} suggest that the 
gap between the lower and upper bounds for \lsr can be completely 
eliminated for constant $w$. 

\vspace{-0.05in}
\begin{theorem}%
    For constant $w$, \lsr can be solved using $O(d\log d)$ 
    actions, agreeing with the $\Omega(d \log d)$ lower bound. 
\end{theorem}
\vspace{-0.15in}
\begin{proof}
    By Lemma~\ref{l:column-sorting}, when $w$ is a constant, a corresponding 
    \clsr problem can be solved using $O(wd\log w) = O(d)$ actions. 
    Lemma~\ref{l:small-w} then applies to do the final sorting in $O(wd\log d) 
    = O(d\log d)$ actions. The total number of required actions is then $O(d\log d)$. 
    The lower bound is given by Corollary~\ref{c:const-w-lower}.
\end{proof}
\vspace{-0.1in}

The following provides a tighter upper bound for non-constant $w$. 
\begin{theorem}%
    An arbitrary \lsr instance can be solved using 
    $O(wd \max\{\log d, \log w\})$ actions. 
\end{theorem}
\begin{proof}
    \hbadness=10000
    Applying Lemma~\ref{l:column-sorting} to the \lsr problem yields a  
    stack sorted instance of \lsr using $O(wd \log w)$ actions. Since sorting
    each of the stacks afterward takes $O(d\log d)$ time, the full 
    \lsr problem can be solved using $O(wd \max \{\log d, \log w\})$ actions. 
\end{proof}

\hbadness=10001% Suppress badness
The greatly improved and general upper bound %$O(wd~\max \{\log~d, \log~w\})$ 
is now fairly close to the general lower bound $\Omega(wd)$ within only a 
logarithmic factor.%

\begin{algorithm}
\begin{small}
	\DontPrintSemicolon
	Poly-C-LSR$(1, w)$\;
    \lFor{$c_i \in [1, w]$}{Sort stack $c_i$}\label{alg:lsr-sort}
    \Return\;

    \SetKwProg{Fn}{Function}{}{}
    \Fn{\normalfont~Poly-C-LSR$(l, r)$}{\label{alg:clsr-begin}
        \lIf{\normalfont~$l = r$}{\Return}
        $i \gets l, j \gets r, m \gets \lfloor (l + r) / 2 \rfloor$\;
        \For{$o_i \in O$}{
            \lIf{$l \leq \pi^1_G(o_i) \leq m$}{label $o_i$: {\em left}}
            \lElseIf{$m + 1 \leq \pi^1_G(o_i) \leq r$}{label $o_i$: {\em right}}
        }
        \lFor{$c_i \in [l, r]$}{Reorder stack $c_i$}
        \lFor{$c_i \in [l, r]$}{Consolidate stack $c_i$}
        Poly-C-LSR$(l, m)$\;\label{alg:clsr-reorder}
        Poly-C-LSR$(m + 1, r)$\;\label{alg:clsr-consolidate}
        \Return\;
    }\label{alg:clsr-end}
	\caption{Poly-LSR$(\arrangement_I, \arrangement_G)$}
	\label{alg:lsr}
\end{small}
\end{algorithm}

The algorithmic process (Poly-LSR) for the method introduced above is shown
in Alg.~\ref{alg:lsr}.  It contains the subroutine for solving \clsr\
(Poly-C-LSR, in line~\ref{alg:clsr-begin}-\ref{alg:clsr-end}), which
iteratively separates the left and right objects in each stack
(line~\ref{alg:clsr-reorder}), and consolidate the stacks
(line~\ref{alg:clsr-consolidate}).  Details are already mentioned in
Lemma~\ref{l:column-sorting}.  \lsr\ is solved by first calling Poly-C-LSR and
then sorting all the stacks using the routine in Lemma~\ref{l:small-w}
(line~\ref{alg:lsr-sort}).  The overall time complexity is $O(wd~\max \{\log~d,
\log~w\})$, which is equivalent to the number of steps in the generated solution.
%Note that every sorting operation that occurs here is actually making the objects to line up
%in the right order, and this order is already known.

%% file: optimal.tex
%%%%%%%%%%%%%%%%%%%%%%%%%%%%%%%%%%%%%%%%%%%%%%%%%%%%%%%%%%%%%%
\section{Optimal and Sub-optimal \lsr Solvers}\label{sec:optimal}
%%%%%%%%%%%%%%%%%%%%%%%%%%%%%%%%%%%%%%%%%%%%%%%%%%%%%%%%%%%%%%

%%%%%%%%%%%%%%%%%%%%%%%%%%%%%%%%%%%%%%%%%%%%%%%%%%%%%%%%%%%%%%
% \subsection{Combinatorial Search Based Solvers}\label{sec:optimal-search}
%%%%%%%%%%%%%%%%%%%%%%%%%%%%%%%%%%%%%%%%%%%%%%%%%%%%%%%%%%%%%%

\lsr~can be reduced to a {\em Shortest Path Problem}, which searches for a
minimum weight path between two nodes in an undirected graph.
Here a node simply denotes an arrangement $\pi$. The neighbors of this
node are all the arrangements reachable from $\pi$ via a single \tpap
action. The edge weights between connected nodes are uniform.

The A* search algorithm\cite{HarNilRap68} is a common tool for solving such a
problem optimally. The branching factor is $(w + 1)w$ since a \tpap
action picks an object from one of $(w + 1)$ stacks, and places it in one of
the other $w$ stacks.  Several heuristic functions are designed to guide the
search:

{\bf Depth Based Heuristic (DBH).}
This heuristic returns an admissible number of \tpapa{}s 
needed to move $o_i \in \objects$ to its goal.
The detailed process appears in Alg.~\ref{alg:dbh}.
It initially checks if $o_i$ is at its goal location (line~\ref{alg:dbh-solved}),
and simply returns $0$ if this statement is true.
Lines~\ref{alg:dbh-count1} and \ref{alg:dbh-count2} calculate 
the number of objects in front of $o_i$ in $\pi_C$ (resp.~$\pi_G$), 
and denote it as $n_c$ (resp.~$n_g$).
At this point (line~\ref{alg:dbh-same-column}), if the object is currently in its goal stack,
DBH computes the estimated number of moves via the following process:
(1) take $o_i$ out of $\pi_G(o_i)$,
(2) make the goal pose reachable by inserting/removing intermediary objects from $\pi_G(o_i)$, and
(3) place $o_i$ back into $\pi_G(o_i)$.
If the object is not in its goal stack, then one of the following apply:
% option a
If $\pi_G(o_i)$ is reachable from $\pi_C(o_i)$ solely by removing intermediary
objects in front of both locations, then the object can be moved in $n_c + n_g
+ 1$ steps;
% Option b
Otherwise, $o_i$ needs to be moved to an intermediate stack and this induces
an extra move.

\begin{algorithm}
\begin{small}
	\DontPrintSemicolon
    \lIf{$\pi_C(o_i) = \pi_G(o_i)$}{\Return~0\label{alg:dbh-solved}}
    \mbox{$n_c \gets |\{o_j | o_j \in \objects, 
    \pi_C^{1}(o_j) = \pi_C^{1}(o_i),
    \pi_C^{2}(o_j) < \pi_C^{2}(o_i)\}|$}\;\label{alg:dbh-count1}
    %\vspace{-1em}
    \mbox{$n_g \gets |\{o_j | o_j \in \objects, 
    \pi_G^{1}(o_j) = \pi_G^{1}(o_i),
    \pi_G^{2}(o_j) < \pi_G^{2}(o_i)\}|$}\;\label{alg:dbh-count2}
    %\vspace{-1em}
    \If{$\pi_C^{1}(o_i) = \pi_G^{1}(o_i)$\label{alg:dbh-same-column}}{
        \lIf{$n_c > n_g$}{\Return~$2n_c - n_g + 2$}
        \lElse{\Return~$n_g + 2$}
    }
    \Else{
        \If{$\pi_C^{2}(o_i) + \pi_G^{2}(o_i) - 1 \leq 
        (w + 1)d - n$}{\Return~$n_c + n_g + 1$}
        \lElse{\Return~$n_c + n_g + 2$}
    }
	\caption{DBH($o_i, \pi_C, \pi_G$)}
	\label{alg:dbh}
\end{small}
\end{algorithm}

The following variants of DBH deal with multiple objects:
\begin{enumerate}
    \item DBH1: admissible, takes the maximum DBH value over all objects:
    $h_{\text{DBH1}} = \max\nolimits_{o \in \objects} \text{DBH}(o, \pi_C, \pi_G)$.
    \item DBHn: inadmissible, takes the summation of DBH values:
    $h_{\text{DBHn}} = \sum\nolimits_{o \in \objects} \text{DBH}(o, \pi_C, \pi_G)$.
\end{enumerate}

{\bf Column Based Heuristic (CBH).}
Described in Alg.~\ref{alg:cbh},
CBH counts the summation of the minimum number of actions 
necessary to move each object to its goal stack.
As opposed to DBH which seeks a tight estimate for a single object,
CBH considers all objects.

The detailed process is as follows.  For every object $o_i \in \objects$, CBH
first determines whether $\pi_C^1(o_i) = \pi_G^1(o_i)$
(line~\ref{alg:cbh-column}).  If $\pi_C^1(o_i) = \pi_G^1(o_i)$, the heuristic
value $h$ remains unchanged if the objects behind $o_i$ are all at their goals.
Otherwise, there exists either an object currently deeper than $o_i$ that needs
to be evacuated, or an object in another stack that needs to be inserted to
$\pi^1_G(o_i)$ at a depth deeper than $\pi^2_G(o_i)$.  Thus $o_i$ must be taken
out of its goal stack.  This requires two additional actions
(line~\ref{alg:cbh-same}).

If $\pi_C^1(o_i) \neq \pi_G^1(o_i)$, 
it takes at least 1 action for $o_i$ to be moved to $\pi^1_G(o_i)$ (line~\ref{alg:cbh-diff1}).
However if the empty locations in the stacks other than 
$\pi^1_C(o_i)$ and $\pi^1_G(o_i)$ cannot contain all
the objects in front of $\pi_C(o_i)$ and $\pi_G(o_i)$, 
$o_i$ must be moved to an intermediate stack.
This requires $2$ actions (line~\ref{alg:cbh-diff2}).

\begin{algorithm}
\begin{small}
	\DontPrintSemicolon
    $h \gets 0$\;
    \For{\normalfont~$o_i \in \objects$}{
        \If{\normalfont $\pi^{1}_{C}(o_i)$ = $\pi^{1}_{G}(o_i)$ \label{alg:cbh-column}}
        {
            \lIf{$\exists o_j \in \objects, \pi^{1}_{C}(o_j) = \pi^{1}_{C}(o_i), 
             \pi^{2}_{C}(o_j) > \pi^{2}_{C}(o_i), \pi_{C}(o_j) \neq \pi_{G}(o_j)$}
             {$h \gets h + 2$}
            \label{alg:cbh-same}
        }
        \Else{
            \lIf{$(w + 1)d - n < \pi^{2}_{C}(o_i) + \pi^{2}_{G}(o_i) - 1$}
            {$h \gets h + 2$ \label{alg:cbh-diff2}}
            \lElse{$h \gets h + 1$ \label{alg:cbh-diff1}}
        }
    }
    \Return~h\;
	\caption{CBH($\pi_C, \pi_G$)}
	\label{alg:cbh}
\end{small}
\end{algorithm}

An example of DBH and CBH calculation is shown in Fig.~\ref{fig:heuristic-example}.
The running time for DBH  is $O(d)$, so totally $O(nd)$ for both DBH1 and DBHn.
CBH runs in $O(n)$ time when dealing the objects in each stack in a bottom-up manner.

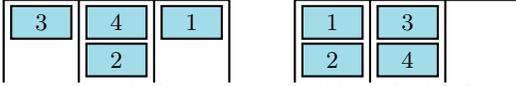
\begin{figure}
    \centering 
    \begin{tabular}{ccc}
        \begin{tikzpicture}
            \foreach \nodeName/\nodeLocation in {3/{(0.5, 1)},4/{(1.5, 1)},2/{(1.5, 0.5)},1/{(2.5, 1)}}{
                \node[thick,shape=rectangle,draw=black,fill=node_fill,text=black,minimum width=0.8cm,minimum height=0.4cm] 
                (\nodeName) at \nodeLocation {\small $\nodeName$}; 
            }
            \draw[-,thick] (0,0.2) to (0,1.3);
            \draw[-,thick] (0,1.3) to (3,1.3);
            \draw[-,thick] (3,1.3) to (3,0.2);
            \draw[-,thick] (2,1.3) to (2,0.2);
            \draw[-,thick] (1,1.3) to (1,0.2);
        \end{tikzpicture}
        &&
        \begin{tikzpicture}
            \foreach \nodeName/\nodeLocation in {1/{(0.5, 1)},2/{(0.5, 0.5)},3/{(1.5, 1)},4/{(1.5, 0.5)}}{
                \node[thick,shape=rectangle,draw=black,fill=node_fill,text=black,minimum width=0.8cm,minimum height=0.4cm] 
                (\nodeName) at \nodeLocation {\small $\nodeName$}; 
            }
            \draw[-,thick] (0,0.2) to (0,1.3);
            \draw[-,thick] (0,1.3) to (3,1.3);
            \draw[-,thick] (3,1.3) to (3,0.2);
            \draw[-,thick] (2,1.3) to (2,0.2);
            \draw[-,thick] (1,1.3) to (1,0.2);
        \end{tikzpicture}
    \end{tabular}
    \caption{An example for heuristic calculation. 
    The left figure denotes $\pi_C$, while the right one denotes $\pi_G$.
    The heuristic values are  
    $h_{\text{DBH1}} = \max \{3, 1, 3, 4\} = 4$, 
    $h_{\text{DBHn}} = 3 + 1 + 3 + 4 = 11$, 
    $h_{\text{CBH}} = 2 + 1 + 2 + 2 = 7$, 
    The optimal solution for this problem costs $9$ steps.}
    \label{fig:heuristic-example}
\end{figure}

An alternate optimal solver is {\em bidirectional heuristic search}
(BHPA)~\cite{Poh69}.  It runs two A* searches simultaneously: One starts from
$\pi_I$ and searches for $\pi_G$; the other starts from $\pi_G$ and searches
for $\pi_I$. BHPA terminates when it finds a path with length $\mu \leq \max\{f_I, f_G\}$.  
Here $f_I$ and $f_G$ denote the minimum $f$-values in the two search fringes, respectively.

By multiplying the heuristic value with a weight $\omega > 1$, 
{\em Weighted A* Search}~\cite{Pea84} generates $\omega$-approximate solutions.
It runs significantly faster than A* search.
Weighted A* is denoted as A*$(\omega)$ and weighted BHPA as BHPA$(\omega)$.

%%%%%%%%%%%%%%%%%%%%%%%%%%%%%%%%%%%%%%%%%%%%%%%%%%%%%%%%%%%%%%
% \subsection{Other Potential Solvers and Their Weaknesses}
%%%%%%%%%%%%%%%%%%%%%%%%%%%%%%%%%%%%%%%%%%%%%%%%%%%%%%%%%%%%%%

\begin{remark}
Other algorithms, including, but not limited to, 
ALT~\cite{GolHar05}, ID~\cite{StaKor11}, CBS~\cite{ShaSte+15}, ILP~\cite{YuLaV16},
although efficient in solving general search or \pmg\ problems, 
are expected to underperform on \lsr\ because of the high density and lack of parallel movements.
Details are omitted due to the space.
\end{remark}

%% file: experiments.tex
%%%%%%%%%%%%%%%%%%%%%%%%%%%%%%%%%%%%%%%%%%%%%%%%%%%%%%%%%%%%%%
\section{Experimental Results}\label{sec:experiments}
%%%%%%%%%%%%%%%%%%%%%%%%%%%%%%%%%%%%%%%%%%%%%%%%%%%%%%%%%%%%%%
%%%%%%%%%%%%%% Set plot style %%%%%%%%%%%%%%%%%%%%%%%%%%%%%%%%
\pgfplotsset{every linear axis/.append style={
    font=\scriptsize,
    % Adjust height and width for every plot
    height = 0.5 * \linewidth,
    width = \linewidth,
    % Make x axis consistent to data
    xtick=data,
    % Methods to change x, y label location
    x label style={at={(axis description cs:0.5,-0.1)}},
    y label style={at={(axis description cs:-0.05,0.5)}},
    % Add a grid
    ymajorgrids=true,
    grid style=dashed
}}
% Change font size for figures
\tikzset{every picture/.style={font=\footnotesize}}
%%%%%%%%%%%%%%%%%%%%%%%%%%%%%%%%%%%%%%%%%%%%%%%%%%%%%%%%%%%%%%
%%%%%%%%%%%%%%%%%%%%%%%%%%%%%%%%%%%%%%%%%%%%%%%%%%%%%%%%%%%%%%
This section presents experimental validation for the algorithms introduced in this paper.
For each problem setup $(w, d, n)$, $100$ randomly generated \lsr\ instances were created.
The experiments were conducted with varying values for $w$, $d$, and $n$.
Unsurprisingly, altering the values of the parameters had little effect 
on the overall problem. For brevity, this section details the specific scenario where 
$w = 5, d = 5,$ and $n$ varies.\footnote{~Detailed results are provided in Appendix~\ref{sec:table}.}

Both the {\em success rate} and {\em average cost} are evaluated for each problem setup.
The success rate is the percentage of instances that generated a solution before 
a five second timeout occurred. 
The quality of solutions is presented as the average number of actions 
$|A|$.%\footnote{Figures for average cost are not provided since the data points are too close to each other.}

All experiments were executed on a 
Intel\textsuperscript{\textregistered} Core\textsuperscript{TM} i7-6900K CPU with 32GB RAM at 2133MHz.

% \subsection{Polynomial Time Solvers}
Polynomial algorithms are tested with naive implementations.
Poly-D is the simple $O(wd^2)$ algorithm in Section~\ref{sec:structure}.
Both the Poly-D and Poly-LSR algorithms are able to solve \lsr\ problems with $1000$ objects in 1 second,
which is already beyond a practical number.
Poly-D generates better solutions when $n$ is low, 
e.g., when $n = 24, |A_{\text{Poly-D}}| = 97.57, |A_{\text{Poly-LSR}}| = 183.97$.
The performance flips when there is more than $1000$ objects.
For example, when $w = 50, d = 40, n = 2000$, 
Poly-D uses  $\sim\hspace*{-.25em}63,000$ steps to solve a problem,
while Poly-LSR uses $\sim\hspace*{-.25em}50,000$ steps.

% \subsection{Comparison Between Heuristics}

The heuristics described in Section~\ref{sec:optimal} are tested with the A* algorithm.
As evidenced in Fig.~\ref{fig:eval-heuristic},
the admissable heuristic CBH has a higher success rate than its competitors.
The entry CBH+DBH1 takes the maximum of the two heuristic values, 
but does not provide better performance. 
This is because of the $O(nd)$ overhead for calculating DBH1.

\begin{figure}[htp]
    \begin{tikzpicture}
        \begin{axis}[
            cycle list={
                {teal,every mark/.append style={fill=teal!80!black},mark=*}, 
                {orange,every mark/.append style={fill=orange!80!black},mark=triangle*}, 
                {cyan!60!black,every mark/.append style={fill=cyan!80!black},mark=square*}, 
                {red!70!white,mark=diamond*}
            },
            height = 0.5 * \linewidth,
            xlabel={Number of Objects},
            ylabel={Success Rate},
            xmin=0, xmax=24, ymin=0, ymax=100
            ]
        \addplot coordinates{
        (2,100)(4,100)(6,100)(8,50)(10,3)(12,0)
        (14,0)(16,0)(18,0)(20,0)(22,0)(24,0)
        };
        \addplot coordinates{
        (2,100)(4,100)(6,100)(8,84)(10,10)(12,0)
        (14,0)(16,0)(18,0)(20,0)(22,0)(24,0)
        };
        \addplot coordinates{
        (2,100)(4,100)(6,100)(8,100)(10,100)(12,94)
        (14,57)(16,20)(18,3)(20,0)(22,0)(24,0)
        };
        \addplot coordinates{
        (2,100)(4,100)(6,100)(8,100)(10,100)(12,94)
        (14,57)(16,20)(18,3)(20,0)(22,0)(24,0)
        };
        \legend{DBH1,DBHn,CBH,CBH+DBH1}
        \end{axis}
    \end{tikzpicture}
    \caption{Success rate of heuristics.}
    \label{fig:eval-heuristic}
\end{figure}
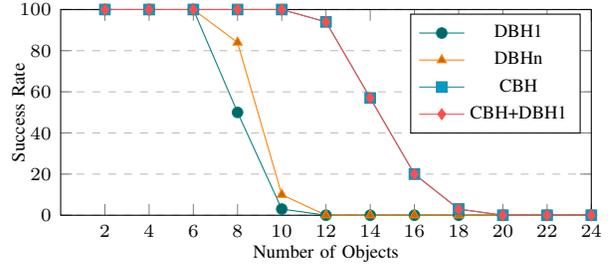

% \subsection{Comparison Between Algorithms}

CBH is used to guide the search algorithms in Section~\ref{sec:optimal}.
Fig.~\ref{fig:eval-algorithm} shows that with the help of CBH, A* runs much
faster than {\em Breath First Search} (BFS) and its bidirectional version
Bi-BFS.  As expected \cite{HarNilRap68}, it also beats BHPA.

The weighted search algorithms generate solutions close to the optimal.
For example, when $n = 10, |A_{\text{A*(2)}}| = 14.69, |A_{\text{BHPA(2)}}| = 14.44$,
while the exact solution $|A_{\text{opt}}| = 13.01$.
BHPA(2) runs faster than A*(2), and also generates better solutions, e.g., 
when $n = 18, |A_{\text{A*(2)}}| = 33.42, |A_{\text{BHPA(2)}}| = 32.92$.
This is because as the heuristic becomes inadmissable,
the termination criterion of BHPA is more easily satisfied.

\begin{figure}[htp]
    \begin{tikzpicture}
        \begin{axis}[
            cycle list={
                {red,every mark/.append style={fill=red!80!black},mark=*}, 
                {blue,every mark/.append style={fill=blue!80!black},mark=o}, 
                {yellow!60!black,every mark/.append style={fill=yellow!80!black},mark=square*}, 
                {black,every mark/.append style={fill=black!80!black},mark=square},
                {brown,mark=star,every mark/.append style={fill=brown!80!black},mark=triangle*},
                {teal,every mark/.append style={fill=teal!80!black},mark=triangle}
            },
            height = 0.5 * \linewidth,
            legend columns=6,
            legend style={at={(0.5,-0.25)},anchor=north},
            xlabel={Number of Objects},
            ylabel={Success Rate},
            xmin=0, xmax=24, ymin=0, ymax=100
            ]
            \addplot coordinates{
            (2,100)(4,100)(6,88)(8,0)(10,0)(12,0)
            (14,0)(16,0)(18,0)(20,0)(22,0)(24,0)
            };
            \addplot coordinates{
            (2,100)(4,100)(6,100)(8,95)(10,5)(12,0)
            (14,0)(16,0)(18,0)(20,0)(22,0)(24,0)
            };
            \addplot coordinates{
            (2,100)(4,100)(6,100)(8,100)(10,100)(12,94)
            (14,57)(16,20)(18,3)(20,0)(22,0)(24,0)
            };
            \addplot coordinates{
            (2,100)(4,100)(6,100)(8,100)(10,100)(12,93)
            (14,55)(16,20)(18,4)(20,0)(22,0)(24,0)
            };
            \addplot coordinates{
            (2,100)(4,100)(6,100)(8,100)(10,100)(12,100)
            (14,100)(16,100)(18,100)(20,94)(22,64)(24,54)
            };
            \addplot coordinates{
            (2,100)(4,100)(6,100)(8,100)(10,100)(12,100)
            (14,100)(16,100)(18,100)(20,99)(22,75)(24,56)
            };
        \legend{BFS,Bi-BFS,A*,BHPA,A*(2),BHPA(2)}
        \end{axis}
    \end{tikzpicture}
    \caption{Success rate of algorithms}
    \label{fig:eval-algorithm}
\end{figure}
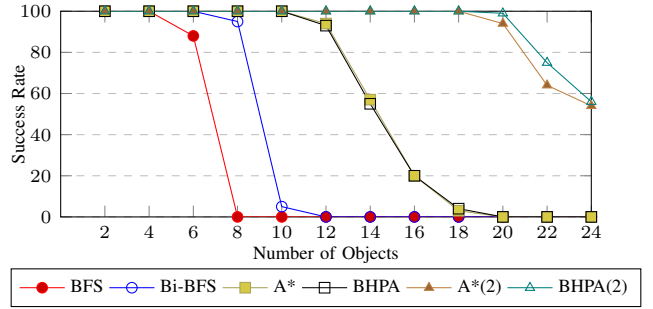

%% file: conclusion.tex
%%%%%%%%%%%%%%%%%%%%%%%%%%%%%%%%%%%%%%%%%%%%%%%%%%%%%%%%%%%%%%
\section{Conclusion}\label{sect:conclusion}
%%%%%%%%%%%%%%%%%%%%%%%%%%%%%%%%%%%%%%%%%%%%%%%%%%%%%%%%%%%%%%
%%%%%%%%%%%%%%%%%%%%%%%%%%%%%%%%%%%%%%%%%%%%%%%%%%%%%%%%%%%%%%
% Recap - Trying to keep it concise ... and failing
This paper describes a novel approach to the object rearrangement problem where 
objects are stored in stack-like containers. Fundamental optimality 
bounds are provided 
by modeling these challenges as pebble motion problems. 
While optimal solvers exist to tackle pebble motion problems, these methods 
are ill-suited for the stack object rearrangement problem due to the scalability of 
such approaches. To overcome this shortcoming, an algorithmic solution is presented 
that produces sub-optimal solutions albeit much faster than optimal solvers. The
utility of the proposed method is validated via extensive experimental evaluation. 

% Emphasize Multiarm scenario
It is not immediately clear whether the techniques presented herein for solving
the labeled stack rearrangement problem with a single manipulator can be
directly extended to the multi-arm scenario.  For example, feasibility tests will be
complicated by the need to reason in the joint configuration space of the
manipulator arms in order to generate collision-free trajectories for the
manipulators.  Whereas a single manipulator performs a series of sequential
\tpapa{}s, a multi-arm setup \textit{may} permit some level of
parallelizability which may have consequential effects on the optimal solution
for specific scenarios.

%% file: appendices.tex
\onecolumn
\section{Detailed Experimental Results}\label{sec:table}
% Column colors
\newcolumntype{g}{>{\columncolor[RGB]{224,255,224}}c}
\newcolumntype{y}{>{\columncolor[RGB]{255,255,224}}c}
\newcolumntype{b}{>{\columncolor[RGB]{224,224,255}}c}
\newcolumntype{r}{>{\columncolor[RGB]{255,224,224}}c}
%%%%%%%%%%%%%%%%%%%%%%%%%%%%%%%%%%%%%%%%%%%%%%%%%%%%%%%%%%%%%%
%%%%%%%%%%%%%%%%%%%%%%%%%%%%%%%%%%%%%%%%%%%%%%%%%%%%%%%%%%%%%%
%%%%%%%%%%%%% HUGE TABLE START %%%%%%%%%%%%%%%%%%%%%%%%%%%%%%%
\begin{table*}[ht]    
    \caption{Average Solution Cost of the Algorithms}
    \setlength\tabcolsep{4.5pt} % default value: 6pt
    \begin{threeparttable}
        \centering
        \begin{tabular}{|c|c|c|g|y|y|b|b|b|b|r|r|r|r|r|r|}
            \hline 
$w$&$d$&$n$	&Opt.Val. &Poly-D &Poly-LSR	&DBH1	&DBHn	&CBH	&CBH+DBH1	&BFS	    &Bi-BFS	&A*	    &BHPA	&A*(2)	&BHPA(2)\\ \hline \hline
2&3&6	&12.11	&14.76	&31.14	&12.11	&12.84	&12.11	&12.11	&12.11	&12.11	&12.11	&12.11	&13.12	&13.49\\ \hline
3&3&9	&17.16	&24.32	&53.95	&16.13*	&17.58*	&17.16	&17.16	&11.17*	&17.16	&17.16	&17.16	&19.8	&19.78\\ \hline
4&3&12	&21.95	&34.25	&83.88	&NA	    &NA	    &21.52*	&21.52*	&NA	    &16.4*	&21.52*	&21.1*	&25.35	&25.2\\ \hline
5&3&15	&NA	    &44.49	&114.85	&NA	    &NA	    &24.72*	&24.72*	&NA	    &NA	    &24.72*	&24.94*	&31.85	&31.32\\ \hline
6&3&18	&NA	    &55.1	&146.53	&NA	    &NA	    &NA	    &NA	    &NA	    &NA	    &NA      &NA	    &37.28	&36.77\\ \hline
7&3&21	&NA	    &66.36	&178.49	&NA	    &NA	    &NA	    &NA	    &NA	    &NA	    &NA      &NA	    &44.43*	&43.98*\\ \hline
8&3&24	&NA	    &76.58	&217.77	&NA	    &NA	    &NA	    &NA	    &NA	    &NA	    &NA      &NA	    &49.73*	&49.82*\\ \hline
9&3&27	&NA	    &88	    &256.47	&NA	    &NA	    &NA	    &NA	    &NA	    &NA	    &NA      &NA	    &56.71*	&56.15*\\ \hline \hline

2&3&6	&12.11	&14.72	&39.52	&12.11	&12.84	&12.11	&12.11	&12.11	&12.11	&12.11	&12.11	&13.12	&13.49\\ \hline
2&4&8	&18.17	&24.25	&58.74	&18.17	&20.47	&18.17	&18.17	&18.17	&18.17	&18.17	&18.17	&20.57	&20.27\\ \hline
2&5&10	&25.08	&35.65	&78.96	&21.48*	&25.23*	&25.08	&25.08	&NA	    &25.02*	&25.08	&25.08	&28.62	&28.4\\ \hline
2&6&12	&32.57  &48.97	&101.2	&NA	    &NA	    &30.23*	&30.2*	&NA	    &25.0*	&30.23*	&29.57*	&37.2	&37.41\\ \hline
2&7&14	&NA	    &64.58	&124.61	&NA	    &NA	    &35.0*	&35.0*	&NA	    &NA	    &35.0*	&34.29*	&46.21*	&46.38\\ \hline
2&8&16	&NA	    &79.91	&147.63	&NA	    &NA	    &36.5*	&36.5*	&NA	    &NA	    &36.5*	&NA	    &53.88*	&54.94*\\ \hline
2&9&18	&NA	    &98.63	&172.42	&NA	    &NA	    &NA	    &NA	    &NA	    &NA	    &NA	    &NA	    &61.6*	&61.51*\\ \hline
2&10&20	&NA	    &118.59	&197.3	&NA	    &NA	    &NA	    &NA	    &NA	    &NA	    &NA	    &NA	    &68.25*	&69.45*\\ \hline \hline

5&5&2	&1.74	&1.86	&15.35	&1.74	&1.74	&1.74	&1.74	&1.74	&1.74	&1.74	&1.74	&1.74	&1.74\\ \hline
5&5&4	&4.2	&5.04	&30.73	&4.2	&4.38	&4.2	&4.2	&4.2	&4.2	&4.2	&4.2	&4.29	&4.26\\ \hline
5&5&6	&6.87	&9.49	&45.82	&6.87	&7.37	&6.87	&6.87	&6.65*	&6.87	&6.87	&6.87	&7.12	&7.18\\ \hline
5&5&8	&9.62	&15.89	&61.31	&8.56*	&10.35*	&9.62	&9.62	&NA	    &9.6*	&9.62	&9.62	&10.3	&10.28\\ \hline
5&5&10	&13.01	&22.63	&76.93	&9.33*	&11.9*	&13.01	&13.01	&NA	    &10.2*	&13.01	&13.01	&14.69	&14.44\\ \hline
5&5&12	&16.01	&30.42	&92.08	&NA	    &NA	    &15.87*	&15.87*	&NA	    &NA	    &15.87*	&15.86*	&18.49	&18.32\\ \hline
5&5&14	&19.74 	&38.86	&107.13	&NA	    &NA	    &18.86*	&18.86*	&NA	    &NA	    &18.86*	&18.85*	&22.79	&22.69\\ \hline
5&5&16	&NA	    &48.1	&122.74	&NA	    &NA	    &21.5*	&21.5*	&NA	    &NA	    &21.5*	&21.65*	&27.54	&27.26\\ \hline
5&5&18	&NA	    &57.53	&138.08	&NA	    &NA	    &23.33*	&23.33*	&NA	    &NA	    &23.33*	&23.25*	&33.42	&32.92\\ \hline
5&5&20	&NA	    &68.17	&153.61	&NA	    &NA	    &NA	    &NA	    &NA	    &NA	    &NA	    &NA	    &38.23*	&37.88*\\ \hline
5&5&22	&NA	    &80.04	&168.41	&NA	    &NA	    &NA	    &NA	    &NA	    &NA	    &NA	    &NA	    &44.16*	&44.12*\\ \hline
5&5&24	&NA	    &97.57	&183.97	&NA	    &NA	    &NA	    &NA	    &NA	    &NA	    &NA	    &NA	    &55.69*	&55.59*\\ \hline
        \end{tabular} 
        \begin{tablenotes} 
            \item[*] Failed instances are not involved in calculating the average cost,
            which makes the data point less informative.
            \item NA: all test cases are failed.
        \end{tablenotes}
    \end{threeparttable}
    \label{tab:evaluation}
\end{table*}

%%%%%%%%%%%%% HUGE TABLE FINISH %%%%%%%%%%%%%%%%%%%%%%%%%%%%%%
%%%%%%%%%%%%%%%%%%%%%%%%%%%%%%%%%%%%%%%%%%%%%%%%%%%%%%%%%%%%%%
%%%%%%%%%%%%%%%%%%%%%%%%%%%%%%%%%%%%%%%%%%%%%%%%%%%%%%%%%%%%%%

In Table~\ref{tab:evaluation}:
\begin{itemize}
    \item The leftmost $3$ columns denote different setups of \lsr.   
    \item Green columns: optimal costs, 
    achieved by running the A* algorithm with the timeout set to $300$ seconds.\\
    \textbf{Note}: These instances are bottlenecked by the memory requirements of the problem.
    \item Yellow columns: polynomial algorithms.
    \item Blue columns: heuristics.
    \item Red columns: search algorithms.
\end{itemize}
\newpage
%%%%%%%%%%%%%%%%%%%%%%%%%%%%%%%%%%%%%%%%%%%%%%%%%%%%%%%%%%%%%%
%%%%%%%%%%%%% FIGURES           %%%%%%%%%%%%%%%%%%%%%%%%%%%%%%
%%%%%%%%%%%%%%%%%%%%%%%%%%%%%%%%%%%%%%%%%%%%%%%%%%%%%%%%%%%%%%
\begin{figure}[!bp]
\subfloat{%
    \begin{tikzpicture}
        \begin{axis}[
            height = 0.25 * \linewidth,
            width = 0.5\linewidth,
            legend columns=4,
            legend style={at={(0.5,-0.3)},anchor=north},
            cycle list={
                {teal,every mark/.append style={fill=teal!80!black},mark=*}, 
                {orange,every mark/.append style={fill=orange!80!black},mark=triangle*}, 
                {cyan!60!black,every mark/.append style={fill=cyan!80!black},mark=square*}, 
                {red!70!white,mark=diamond*}
            },
            xlabel={Number of Slots $\cdot$ Number of Objects},
            ylabel={Success Rate},
            xticklabels={$2\cdot 6$,$3\cdot 9$,$4\cdot 12$,$5\cdot 15$,$6\cdot 18$,$7\cdot 21$,$8\cdot 24$,$9\cdot 27$},
            xmin=2, xmax=9, ymin=0, ymax=100
            ]
        \addplot coordinates{
        (2,100)(3,71)(4,0)(5,0)(6,0)(7,0)
        (8,0)(9,0)};
        \addplot coordinates{
        (2,100)(3,71)(4,0)(5,0)(6,0)(7,0)
        (8,0)(9,0)};
        \addplot coordinates{
        (2,100)(3,100)(4,83)(5,18)(6,0)(7,0)
        (8,0)(9,0)};
        \addplot coordinates{
        (2,100)(3,100)(4,83)(5,18)(6,0)(7,0)
        (8,0)(9,0)};
        \legend{DBH1,DBHn,CBH,CBH+DBH1}
        \end{axis}
    \end{tikzpicture}
}%
\subfloat{%
    \begin{tikzpicture}
        \begin{axis}[
            height = 0.25 * \linewidth,
            width = 0.5\linewidth,
            cycle list={
                {teal,every mark/.append style={fill=teal!80!black},mark=*}, 
                {orange,every mark/.append style={fill=orange!80!black},mark=triangle*}, 
                {cyan!60!black,every mark/.append style={fill=cyan!80!black},mark=square*}, 
                {red!70!white,mark=diamond*}
            },
            legend columns=4,
            legend style={at={(0.5,-0.3)},anchor=north},
            xlabel={Depth $\cdot$ Number of Objects},
            ylabel={Success Rate},
            xticklabels={$3\cdot 6$,$4\cdot 8$,$5\cdot 10$,$6\cdot 12$,$7\cdot 14$,$8\cdot 16$,$9\cdot 18$,$10\cdot 20$},
            xmin=3, xmax=10, ymin=0, ymax=100
            ]
        \addplot coordinates{
        (3,100)(4,100)(5,25)(6,0)(7,0)(8,0)
        (9,0)(10,0)};
        \addplot coordinates{
        (3,100)(4,100)(5,22)(6,0)(7,0)(8,0)
        (9,0)(10,0)};
        \addplot coordinates{
        (3,100)(4,100)(5,100)(6,57)(7,12)(8,2)
        (9,0)(10,0)};
        \addplot coordinates{
        (3,100)(4,100)(5,100)(6,56)(7,12)(8,2)
        (9,0)(10,0)};
        \legend{DBH1,DBHn,CBH,CBH+DBH1}
        \end{axis}
    \end{tikzpicture}
}
    % \begin{tikzpicture}
    %     \begin{axis}[
    %         cycle list={
    %             {teal,every mark/.append style={fill=teal!80!black},mark=*}, 
    %             {orange,every mark/.append style={fill=orange!80!black},mark=triangle*}, 
    %             {cyan!60!black,every mark/.append style={fill=cyan!80!black},mark=square*}, 
    %             {red!70!white,mark=diamond*}
    %         },
    %         legend columns=4,
    %         legend style={at={(0.5,-0.3)},anchor=north},
    %         xlabel={Number of Objects},
    %         ylabel={Success Rate},
    %         xmin=0, xmax=24, ymin=0, ymax=100
    %         ]
    %     \addplot coordinates{
    %     (2,100)(4,100)(6,100)(8,50)(10,3)(12,0)
    %     (14,0)(16,0)(18,0)(20,0)(22,0)(24,0)
    %     };
    %     \addplot coordinates{
    %     (2,100)(4,100)(6,100)(8,84)(10,10)(12,0)
    %     (14,0)(16,0)(18,0)(20,0)(22,0)(24,0)
    %     };
    %     \addplot coordinates{
    %     (2,100)(4,100)(6,100)(8,100)(10,100)(12,94)
    %     (14,57)(16,20)(18,3)(20,0)(22,0)(24,0)
    %     };
    %     \addplot coordinates{
    %     (2,100)(4,100)(6,100)(8,100)(10,100)(12,94)
    %     (14,57)(16,20)(18,3)(20,0)(22,0)(24,0)
    %     };
    %     \legend{DBH1,DBHn,CBH,CBH+DBH1}
    %     \end{axis}
    % \end{tikzpicture}
    \caption{Success rate of heuristics.}
    \label{fig:eval-heuristic-all}
\end{figure}

\begin{figure}[!bp]
\subfloat{%
    \begin{tikzpicture}
        \begin{axis}[
            height = 0.25 * \linewidth,
            width = 0.5\linewidth,
            legend columns=3,
            legend style={at={(0.5,-0.3)},anchor=north},
            cycle list={
                {red,every mark/.append style={fill=red!80!black},mark=*}, 
                {blue,every mark/.append style={fill=blue!80!black},mark=o}, 
                {yellow!60!black,every mark/.append style={fill=yellow!80!black},mark=square*}, 
                {black,every mark/.append style={fill=black!80!black},mark=square},
                {brown,mark=star,every mark/.append style={fill=brown!80!black},mark=triangle*},
                {teal,every mark/.append style={fill=teal!80!black},mark=triangle}
            },
            xlabel={Number of Slots $\cdot$ Number of Objects},
            ylabel={Success Rate},
            xticklabels={$2\cdot 6$,$3\cdot 9$,$4\cdot 12$,$5\cdot 15$,$6\cdot 18$,$7\cdot 21$,$8\cdot 24$,$9\cdot 27$},
            xmin=3, xmax=10, ymin=0, ymax=100
            ]
        \addplot coordinates{
        (3,100)(4,6)(5,0)(6,0)(7,0)(8,0)
        (9,0)(10,0)};
        \addplot coordinates{
        (3,100)(4,100)(5,5)(6,0)(7,0)(8,0)
        (9,0)(10,0)};
        \addplot coordinates{
        (3,100)(4,100)(5,83)(6,18)(7,0)(8,0)
        (9,0)(10,0)};
        \addplot coordinates{
        (3,100)(4,100)(5,69)(6,17)(7,0)(8,0)
        (9,0)(10,0)};
        \addplot coordinates{
        (3,100)(4,100)(5,100)(6,100)(7,100)(8,99)
        (9,95)(10,92)};
        \addplot coordinates{
        (3,100)(4,100)(5,100)(6,100)(7,100)(8,98)
        (9,98)(10,95)};
        \legend{BFS,Bi-BFS,A*,BHPA,A*(2),BHPA(2)}
        \end{axis}       
    \end{tikzpicture}
}
\subfloat{%
    \begin{tikzpicture}
        \begin{axis}[
            height = 0.25 * \linewidth,
            width = 0.5\linewidth,
            cycle list={
                {red,every mark/.append style={fill=red!80!black},mark=*}, 
                {blue,every mark/.append style={fill=blue!80!black},mark=o}, 
                {yellow!60!black,every mark/.append style={fill=yellow!80!black},mark=square*}, 
                {black,every mark/.append style={fill=black!80!black},mark=square},
                {brown,mark=star,every mark/.append style={fill=brown!80!black},mark=triangle*},
                {teal,every mark/.append style={fill=teal!80!black},mark=triangle}
            },
            legend columns=3,
            legend style={at={(0.5,-0.3)},anchor=north},
            xlabel={Depth $\cdot$ Number of Objects},
            ylabel={Success Rate},  
            xticklabels={$3\cdot 6$,$4\cdot 8$,$5\cdot 10$,$6\cdot 12$,$7\cdot 14$,$8\cdot 16$,$9\cdot 18$,$10\cdot 20$},
            xmin=3, xmax=10, ymin=0, ymax=100
            ]
        \addplot coordinates{
        (3,100)(4,100)(5,0)(6,0)(7,0)(8,0)
        (9,0)(10,0)};
        \addplot coordinates{
        (3,100)(4,100)(5,99)(6,6)(7,0)(8,0)
        (9,0)(10,0)};
        \addplot coordinates{
        (3,100)(4,100)(5,100)(6,57)(7,12)(8,2)
        (9,0)(10,0)};
        \addplot coordinates{
        (3,100)(4,100)(5,100)(6,47)(7,7)(8,0)
        (9,0)(10,0)};
        \addplot coordinates{
        (3,100)(4,100)(5,100)(6,100)(7,98)(8,83)
        (9,67)(10,36)};
        \addplot coordinates{
        (3,100)(4,100)(5,100)(6,100)(7,100)(8,90)
        (9,69)(10,38)};
        \legend{BFS,Bi-BFS,A*,BHPA,A*(2),BHPA(2)}
        \end{axis}
    \end{tikzpicture}
}
    \caption{Success rate of algorithms}
    \label{fig:eval-algorithm-all}
\end{figure}